%% file: main.tex
\documentclass{article}

\usepackage{microtype}
\usepackage{graphicx}
\usepackage{subcaption}
\usepackage{booktabs} 
\usepackage{xspace}
\usepackage{hyperref}

\usepackage[preprint]{icml2026}

\input{math_commands}

\usepackage{amsthm, amssymb, dsfont}
\usepackage{hyperref}
\usepackage{booktabs}
\hypersetup{
  colorlinks=true,
  linkcolor=blue!60!black,
  urlcolor=blue!70!black,
  citecolor=blue!60!black
}
\usepackage[table]{xcolor}
\usepackage{arydshln}
\usepackage{enumitem}
\usepackage{listings}
\usepackage{fontawesome5}

\theoremstyle{plain}
\newtheorem{theorem}{Theorem}[section]
\newtheorem{proposition}[theorem]{Proposition}

\theoremstyle{definition}

\theoremstyle{remark}

\newcommand{\ourmethod}{\textsc{Loscar}\xspace}

\icmltitlerunning{Residual-Based Causal Discovery via Pretrained In-Context~Quantile Regression}

\begin{document}

\twocolumn[
  \icmltitle{Residual-Based Causal Discovery via Pretrained In-Context~Quantile Regression}



  \icmlsetsymbol{equal}{*}

  \begin{icmlauthorlist}
    \icmlauthor{Daniel Klippert}{tudo,rc}
    \icmlauthor{Quang Nhat Huynh}{wwu}
    \icmlauthor{Alexander Marx}{tudo,rc}
  \end{icmlauthorlist}
  \begin{center}
    \small
    \faGithub\;
    \href{https://github.com/causality-lab/loscar}{https://github.com/causality-lab/loscar}
  \end{center}

  \icmlaffiliation{tudo}{Department of Statistics, TU Dortmund University, Germany}
  \icmlaffiliation{rc}{Research Center Trustworthy Data Science and Security, University Alliance Ruhr, Germany}
  \icmlaffiliation{wwu}{Department of Economics, University of Münster, Germany}

  \icmlcorrespondingauthor{Daniel Klippert}{daniel.klippert@tu-dortmund.de}
  \icmlcorrespondingauthor{Alexander Marx}{alexander.marx@tu-dortmund.de}

  \icmlkeywords{Machine Learning, Causal Discovery, Functional Causal Models, Heteroscedasticity, Quantile Regression, In-Context Regression}

  \vskip 0.3in
]



\printAffiliationsAndNotice{}

\begin{abstract}
  Residual-based causal discovery assumes that each variable is generated from its parents and an independent noise source. It estimates this noise and tests its independence from candidate causes. For identifiable functions such as additive-noise models (ANMs) and the broader class of location-scale noise models (LSNMs), this principle identifies the causal direction, as only the causal model admits this independence. Existing LSNM methods rely on parametric assumptions to obtain standardized residuals, risking model misspecification. In contrast, the RESIT-framework allows for nonparametric estimation of residuals to iteratively identify sink nodes, but was proposed for ANMs. To close this gap, we propose \ourmethod, an extension of RESIT to LSNMs that estimates the conditional mean and scale via quantile regression, enabling pretrained tabular foundation models as plug-in estimators. We compare it with a feasible generalized least squares (FGLS) variant, which is consistent given suitable basis functions. Across bivariate and multivariate causal-discovery benchmarks, in-context quantile regression substantially outperforms FGLS with flexible function classes, particularly for multivariate causes.
\end{abstract}

\section{Introduction}

Causal discovery provides a rigorous framework to determine under which conditions causal relations can be derived from observational data~\citep{pc,peters:17:book}. Provided that all relevant sources of variation are observed, i.e., it can be assumed that confounding and selection bias are absent, and no cyclic relations exist, the causal graph can be described as a directed acyclic graph (DAG).

Constraint-based approaches rely on conditional independence testing to recover the Markov equivalence class~\citep{pc}, whereas methods restricting the function class of the structural causal model identify the unique DAG. Early approaches showed identifiability of the DAG under linear non-Gaussian additive models~\citep[LiNGAM]{LiNGAM}. Both the linearity assumption~\citep{hoyer:09:nonlinear} and the additivity assumption have been relaxed by subsequent works that prove identification results for post-nonlinear~\citep[PNL]{PNL} and location-scale noise models~\citep[LSNMs]{LOCI,GRCI}.

Alongside this line of research, another promising direction involves amortized approaches to causal discovery~\citep{lorch:22:amortized-cd,helby:22:supervised-cd}, which treat causal discovery as a supervised machine learning problem by training on a large corpus of synthetically generated data. Benefiting from the advantages of transformer models~\citep{thompson:26:arrow} or the success of tabular foundation models~\citep{xu:26:in-context}, these approaches are scalable and show strong empirical results. These works aim to encode known identification results through drawing the training data from such data generative models, however, they lack formal guarantees and require extensive pretraining.

\paragraph{Our Contributions} Taking a different approach, we integrate pretrained tabular foundation models~\citep{hollmann:23:tabpfn} as plug-in estimators. Our method, \ourmethod, extends the well-studied RESIT algorithm~\citep{RESIT} to LSNMs. In contrast to previous work on LSNMs that rely on distributional assumptions on the noise variable~\citep{LOCI,duong:2023:heteroscedastic,ROCHE,klippert:26:skew} or symmetric noise variables~\citep{skewscore} we show that nonparametric estimation is feasible by estimating the conditional standard deviation via quantile regression. This result enables us to harness the predictive capabilities of pretrained tabular foundation models such as TabPFN~\citep{hollmann:23:tabpfn} through in-context quantile regression without the need of retraining or fine-tuning the model. We propose two strategies through quantile regression and compare them to a variant using feasible GLS~\citep{greene:2000:econometric} which given appropriate basis functions consistently estimates the conditional mean and standard deviation without distributional assumptions. Through empirical evaluation on 19 bivariate benchmarks we show that \ourmethod using in-context quantile regression achieves state-of-the-art performance and outperforms the FGLS variant on flexible function classes (Figure~\ref{fig:radar-comparision-loscar-variants}). Further, we evaluate \ourmethod on multivariate simulated and real-world benchmarks obtaining competitive results while being scalable in low and medium dimensional settings.

We first review LSNMs and their estimation through conditional mean and variance in Section~\ref{sec:preliminaries}, followed by a description of \ourmethod. In Section~\ref{sec:QR}, we derive two approaches that enable in-context quantile regression and describe the FGLS variant in Section~\ref{sec:FGLS}. We discuss related work in Section~\ref{sec:related-work}, present experiments in Section~\ref{sec:experiments}, and conclude thereafter.

\begin{figure}
    \hspace{2.2em}
    \includegraphics[width=0.83\linewidth]{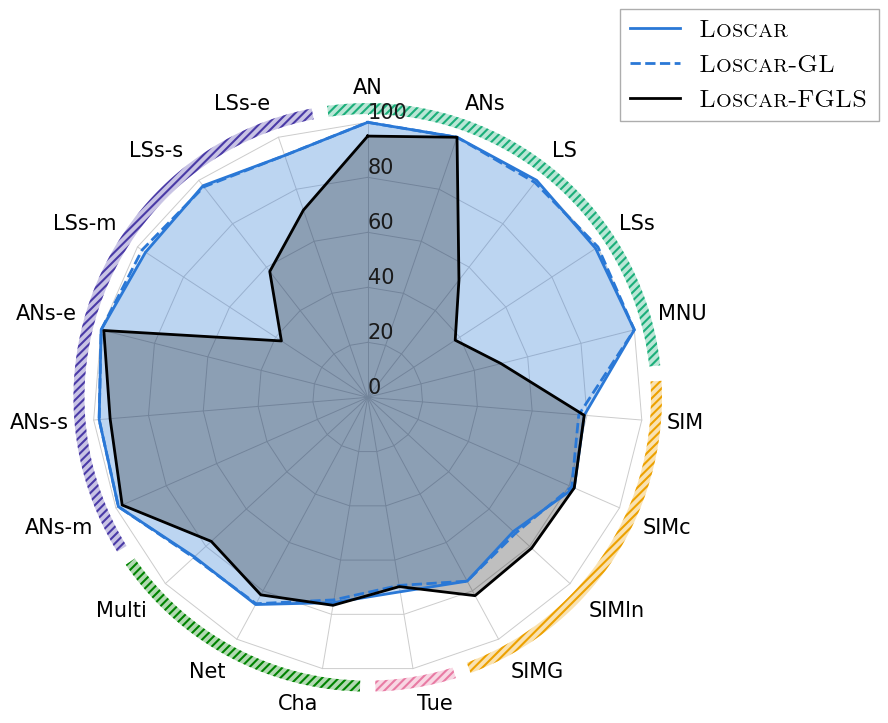}
    \caption{Accuracy in \% over 19 bivariate benchmarks. \ourmethod and \textsc{\ourmethod-GL} use TabPFN for in-context quantile regression, computing the conditional scale using different variants (Section~\ref{sec:QR}). \textsc{\ourmethod-FGLS} (Section~\ref{sec:FGLS}) employs spline-based FGLS. \ourmethod has an average accuracy of $90.04 \%$. While \ourmethod-GL performs on par, the FGLS variant struggles to model complex location-scale relationships.}
    \label{fig:radar-comparision-loscar-variants}
\end{figure}

\section{Residual-Based Causal Discovery} \label{sec:preliminaries}

Below we describe the framework for residual-based causal discovery. First we review LSNMs, which include ANMs as a special case, and their identification results. Then, we discuss the basis of bivariate causal discovery via residual-based independence testing and how these ideas extend to the multivariate setting.

\subsection{Location-Scale Noise Models}
\label{sec:lsnms-intro}

When latent confounding, cycles, and selection are absent, we can model the causal network as a directed acyclic graph (DAG). A causal DAG model is a tuple $(\gG, P_\mX)$, where $\gG$ defines a DAG and $P_\mX$ a probability distribution associated with the nodes of $\gG$, that has been generated through a structural causal model (SCM). In particular, an SCM \(\gC = (\bm{S}, P_{\bm{N}})\) is defined as a collection of assignments \(\bm{S}\): 
\begin{align*}
\label{eq:scm}
X^{(j)} := f^{(j)}(\mathrm{PA}_{X^{(j)}}^{\gG}, N_{X^{(j)}}), \quad \quad j = 1, \dots, d,
\end{align*}
where all \(N_{X^{(j)}}\) (exogenous or noise variables) are jointly independent so that \(P_\mN\) is a product distribution. The DAG \(\gG = (\mV, \gE)\) of an SCM is obtained by creating a vertex \(j \in \mV\), with $|\mV| = d$, for each \(X^{(j)}\) and drawing directed edges \((i,j) \in \gE\) from each parent \(i \in \mathrm{PA}_{X^{(j)}}^{\gG}\) to \(j\). Without restricting the function class, recovery of the DAG $\gG$ is only feasible up to its Markov equivalence class~\citep{peters:17:book}, for example, through assuming faithfulness.

In this paper, we consider the class of \emph{location-scale noise models}~\citep{LOCI}. That is, an SCM is called a location-scale noise model (LSNM) if, for each variable $X^{(j)}$ with non-empty parent set $\mathrm{PA}_{X^{(j)}}^{\mathcal G} \neq \emptyset$, the structural equation of $X^{(j)}$ is given by
\begin{equation} \label{eq:Def1}
X^{(j)} := f^{(j)}(\mathrm{PA}_{X^{(j)}}^{\mathcal G})+g^{(j)}(\mathrm{PA}_{X^{(j)}}^{\mathcal G}) N_{X^{(j)}},
\end{equation}
where $f^{(j)}$ and $g^{(j)}$ are deterministic real-valued functions with $g^{(j)}(\cdot)>0$. If $\mathrm{PA}_{X^{(j)}}^{\mathcal G} = \emptyset$, then $X^{(j)} = N_{X^{(j)}}$. LSNMs naturally extend the class of additive noise models (ANMs)~\citep{kano:03:anms,hoyer:09:nonlinear} and reduce to an ANM when $g^{(j)}(\cdot) \equiv 1$.

Identifiability of the causal graph under LNSMs has been studied by \citet{LOCI,GRCI} for bivariate models. \citet{GRCI} extend these results to the multivariate setting. They conclude that identifiablitly holds up to pathological cases. Related results cover the exact identifiability conditions for Gaussian noise~\citep{khemakhem:21:autoregressive}, uniform and continuous Bernoulli~\citep{sunschulte}, and exponential families~\citep{Expfam}. Below, we discuss how these results can be exploited by practical estimators.

\subsection{Cause-Effect Inference by Residual-Based Independence Testing} 
\label{subsec:residual-indep}

We start by describing the bivariate setting in which we call the two involved random variables $X$ and $Y$. If the data has been generated according to an LSNM encoding the graph $X \to Y$, i.e., $Y := f(X) + g(X) N_Y$, it holds that, by assumption,
\(
    (Y - f(X))/g(X) = N_Y
\)
is independent of $X$.
If the LSNM is identifiable, such an independence relation does not exist for the anti-causal direction, i.e., we cannot find functions $h$ and $k$ such that $Y$ is independent of $N_X=(X - h(Y))/k(Y)$. This approach has its roots in the RESIT (short for regression with subsequent independence testing) framework~\citep{RESIT} that has been developed for ANMs. As in ANMs $g(\cdot) \equiv 1$, RESIT approximates the noise terms by the residuals obtained through standard regression for both hypothetical causal models and decides for the model which does not reject the null hypothesis of independence between noise and hypothesized cause. If independence is rejected in both causal directions, the direction with the higher $p$-value is selected in cases where a decision is forced.\!\footnote{Such a result, however, rather indicates confounding or model misspecification~\citep{schultheiss:24:jmlr-well-specification}.} 

In LSNMs, the noise is typically extracted through estimating $f$ and $g$. When the noise follows a standard Gaussian distribution, \citet{LOCI} provide a weakly consistent spline-based estimator. To relax the Gaussianity assumptions, subsequent work has developed estimators that are more robust to heavy-tailed~\citep{ROCHE} or skewed noise~\citep{klippert:26:skew}, which led to performance improvements on standard benchmarks. In this paper, we build upon the result of \citet{sunschulte}, who show that \emph{consistently estimating mean and standard deviation is sufficient.} We provide a derivation of this result below and propose appropriate estimators in Section~\ref{sec:estimation}.

\paragraph{Standardized Residuals} Given an LSNM for $X \to Y$, as defined in Eq.~\eqref{eq:Def1}, then~\citep{sunschulte}
\begin{equation}
    \label{eq:estimating-ny-scaled}
    \frac{Y - \E[Y|X]}{\mathrm{Sd}(Y|X)} =: \hat{N}_Y \Indep X \; .
\end{equation}

In other words, if we had a consistent estimator for $E[Y|X]$ and $\mathrm{Sd}(Y|X)$ the estimated residuals $\hat{N}_Y$ are independent of cause $X$. The implication of this result is that such an estimator would be applicable to most continuous noise distributions, since we only require that the first two moments of $N_Y$ exist.
In their paper, \citet{sunschulte} discuss the advantages of independence-based approaches in this context, but do not derive a general estimator for the above quantities.

Before introducing our proposed estimators, we briefly sketch the proof idea to provide more intuition to the result in Eq.~\eqref{eq:estimating-ny-scaled}. We first derive that
\begin{align*}
    \E[Y|X] &= f(X) + g(X) \E[N_Y], \\
    \mathrm{Sd}(Y|X) &= g(X) \mathrm{Sd}(N_Y) \; .
\end{align*}
Note that an estimator that consistently learns $\E[Y|X]$ will not coincide with $f(X)$ if $\E[N_Y] \neq 0$. Similarly, $\mathrm{Sd}(Y|X)$ may differ from $g(X)$ whenever $\mathrm{Sd}(N_Y) \neq 1$. Importantly, by estimating the noise variable through Eq.~\eqref{eq:estimating-ny-scaled}, all terms that depend on $X$ cancel and we are guaranteed to learn the standardized version of $N_Y$ in the large sample limit~\citep{sunschulte}, i.e.,
\[
    \frac{Y - \E[Y|X]}{\mathrm{Sd}(Y|X)} = \frac{g(X) (N_Y - \E[N_Y])}{g(X) \mathrm{Sd}(N_Y)} = \frac{N_Y - \E[N_Y])}{\mathrm{Sd}(N_Y)} \; ,
\]
which we will exploit in our approach. Alternatively, Lemma 4.1 by~\citet{sunschulte} clarifies this result by showing that for any LSNM, as defined in Eq.~\eqref{eq:Def1}, there exists a likelihood-equivalent LSNM \begin{align*}
    Y = \mu(X) + \sigma(X) \tilde{N}_Y, \quad \tilde{N}_Y=\frac{N_Y-\E[N_Y]}{\text{Sd}(N_Y)} \Indep X,
\end{align*}
under which $\mu(X) = \E[Y|X]$ and $\sigma(X) = \mathrm{Sd}(Y|X)$. We further note, that these results directly translate to the multivariate setting, where $X$ is replaced by $\mX$.

\subsection{Multivariate Causal Discovery}

To infer DAGs from an iid sample of a multivariate random vector $\mX$, we adapt the RESIT~\citep{RESIT} algorithm to LSNMs as shown in Algorithm~\ref{alg:resti-ls}. We call the resulting algorithm \ourmethod (short for Location-Scale RESIT). Similar adaptations to LSNMs have been explored by~\citet{Expfam,khan:26:beyond}.

In summary, akin to the original RESIT method, Algorithm~\ref{alg:resti-ls} consists of two phases. Phase~1 learns a causal ordering $\pi$ in which the parent set for each node $k$ consists of all nodes that appear prior to $k$ in the causal ordering. The causal ordering is defined as a bijective mapping \(\pi: \{ 1, \dots, d \} \to \{ 1, \dots, d \}\) satisfying
\[\pi(i) < \pi(j) \text{ if } j \in \mathrm{De}_{X^{(i)}}^{\mathcal G},\]
where $\mathrm{De}_{X^{(i)}}^{\mathcal G}$ refers to the descendants of a node $i$. We start with an empty ordering $\pi$ and a candidate set $\gS = \{ 1, \dots, d \}$. Then, we iteratively identify sink nodes, which we add to $\pi$ and remove from $\gS$ until $\gS$ is empty. At each step, we estimate the standardized residuals for each node $k \in \gS$ by obtaining the conditional mean and standard deviation of $X^{(k)} \mid \{ X^{(i)} \}_{i \in S \backslash \{ k\}}$ based on one of the estimators described in Section~\ref{sec:estimation}. Note that both the mean and the standard deviation can be non-constant and are modeled through $\{ X^{(i)} \}_{i \in S \backslash \{ k\}}$. We omit this dependence in the notation of Algorithm~\ref{alg:resti-ls} to avoid clutter. The sink node is then determined as the node for which the estimated $p$-value between standardized residuals and corresponding parent set $\{ X^{(i)} \}_{i \in S \backslash \{ k\}}$ is largest.

\begin{algorithm}[tb]
    \caption{\ourmethod}
    \label{alg:resti-ls}
    \begin{algorithmic}[1]
    \STATE \textbf{Input:} iid sample of $d$-dimensional $\mX$, \\ \phantom{\textbf{Input:}} significance level $\alpha$
    \STATE \textcolor{gray}{\# Phase 1: Causal order search} 
    \STATE $\gS = \{1, \dots, d \}$, $\pi = [ \; ]$ 
     \WHILE{$|\gS| > 0$}
        \FOR{$k \in \gS$}
          \STATE  $\hat{\mu}_k, \hat{\sigma}_k = \mathrm{Estimator}(X^{(k)} \mid \{ X^{(i)} \}_{i \in \gS \backslash \{ k\}})$ 
         \STATE   $p_k = \mathrm{HSIC}\left(\{ X^{(i)} \}_{i \in \gS \backslash \{ k\}}, \frac{X^{(k)} - \hat{\mu}_k}{\hat{\sigma}_k} \right)$
        \ENDFOR
      \STATE  Let $p^*$ be the largest $p$-value. 
      \STATE  $\gS = \gS \backslash \{ k^* \}$ 
       \STATE $\mathrm{PA}(k^*) = \gS$ 
      \STATE  $\pi  = [k^*, \pi]$
    \ENDWHILE
    \STATE \textcolor{gray}{\# Phase 2: Pruning} \\
    \FOR{$k=2, \dots, d$}
        \FOR{$\ell \in \mathrm{PA}(\pi(k))$}
           \STATE $R =\mathrm{PA}(\pi(k)) \backslash \{ \ell \}$ 
           \STATE $\hat{\mu}, \hat{\sigma} = \mathrm{Estimator}(X^{(\pi(k))} \mid \{ X^{(i)} \}_{i \in R})$ 
         \STATE   $p = \mathrm{HSIC}\left(\{ X^{(i)} \}_{i \in \{ \pi(1), \dots, \pi(k-1) \}}, \frac{X^{(\pi(k))} - \hat{\mu}}{\hat{\sigma}} \right)$ 
            \IF{$p > \alpha$}
               \STATE Remove $\ell$ from $\mathrm{PA}(\pi(k))$ 
            \ENDIF
        \ENDFOR 
    \ENDFOR
    \STATE \textbf{Output:} $[\mathrm{PA}(\pi(1)), \dots, \mathrm{PA}(\pi(d))]$
    \end{algorithmic}
\end{algorithm}

Phase 2 of the algorithm prunes superfluous edges for each node $k$ by iteratively checking if removing a member of the current parent set will induce a dependence between the estimated standardized residuals and the nodes appearing before $k$ in the ordering $\pi$.

\paragraph{Consistency} The original RESIT algorithm is guaranteed to find the correct causal graph $\gG$ for identifiable ANMs under causal minimality, given an independence oracle and assuming that the noise can be consistently estimated.
In Algorithm~\ref{alg:resti-ls}, we only modified the procedure for estimating the noise terms and use the HSIC test. Thus, we need to assume that the data has been generated by an identifiable LSNM (cf.~Section~\ref{sec:lsnms-intro}) and we exploit the result that it is sufficient to recover the standardized noise variables, as can be achieved by the multivariate extension to Eq.~\eqref{eq:estimating-ny-scaled}. Further, note that consistent estimation with HSIC can also be achieved via sample splitting if the estimator to extract the noise term is weakly consistent, or without sample splitting if a strongly consistent estimator is available~\citep[Thm.~20]{mooij:16:pairs}.

\section{Estimating Methods}
\label{sec:estimation} 
In this section, we describe three approaches for estimating the conditional mean and scale that we employ within \ourmethod to extract residuals in LSNMs. We propose two methods that estimate the scale from a pretrained in-context quantile regression model (Section~\ref{sec:QR}). Specifically, we use TabPFN~\citep{hollmann:23:tabpfn}, a tabular foundation model, to obtain estimates of the conditional mean and conditional quantile function. The first two approaches then recover the scale from the estimated conditional quantile function in different ways.
The third approach (Section~\ref{sec:FGLS}) jointly estimates the mean and variance functions via feasible generalized least squares, which is consistent under suitable modeling assumptions.

\subsection{Pretrained In-Context Quantile Regression} \label{sec:QR}
Following the argumentation in Section \ref{subsec:residual-indep}, we assume without loss of generality that the LSNM is given by \begin{align} Y=\mu(\bm{X}) + \sigma(\bm{X}) N_Y, \quad N_Y \Indep X, \label{LSNM01} \end{align} where $\mu(\bm{X})=\E[Y|\bm{X}]$ and $\sigma(\bm{X})=\sqrt{\text{Var}(Y|\bm{X})}$. We obtain an estimate of $\mu$ directly from TabPFN and recover $\sigma$ from the quantile function of $Y|\mX = \bm{x}$. For a random variable $Z$ with cumulative distribution function (CDF) $F_Z(z) = \text{P}(Z\leq z)$, the quantile function is given by $$Q_Z(\tau) = \inf\{ z: F_Z(z) \geq \tau\}, \quad \tau \in (0,1).$$
Under the assumption of symmetric noise, mean and median of $Y|\mX=\bm{x}$ coincide, in which case the residuals can be extracted solely based on the quantile function when estimating the scale based on $Q_{Y|\mX=\bm{x}}$. 

\subsubsection{Quantile-Difference Scale Estimation} \label{sec:QR-QD}
As shown in Proposition \ref{prop:sigma}, knowledge of two quantiles of $N_Y$ and the quantile function $Q_{Y|\mX=\bm{x}}$ allows us to recover $\sigma(\bm{x})$ exactly. More importantly, $\sigma(\bm{x})$ can be recovered up to scale without knowledge of the quantiles of $N_Y$. This is sufficient to recover the residuals up to scale within \ourmethod under the LSNM assumption, as for any constant $c \neq 0$ we have that $$(Y-\mu(\mX))/(c \sigma(\bm{x})) \Indep \bm{X}.$$ In the implementation, we calculate the quantiles of $N_Y$ under the assumption that $N_Y \thicksim \mathcal{N}(0,1)$. This choice, however, does not require the true distribution of $N_Y$ to be Gaussian, as any choice of $Q_{N_Y}(\tau_1)-Q_{N_Y}(\tau_2)$ only fixes the otherwise irrelevant multiplicative constant.
\begin{proposition} \label{prop:sigma}
Let $\tau_1, \tau_2 \in (0,1)$. Under the LSNM in Equation \ref{LSNM01}, we have that
\begin{align*}
    \sigma(\bm{x}) &= \dfrac{ Q_{Y|\mX=\bm{x}}(\tau_1) -  Q_{Y|\mX=\bm{x}}(\tau_2)}{Q_{N_Y} (\tau_1) - Q_{N_Y}(\tau_2)} \\ &\propto Q_{Y|\mX=\bm{x}}(\tau_1) -  Q_{Y|\mX=\bm{x}}(\tau_2).
\end{align*}
\end{proposition} 

 \begin{proof}[\textbf{Proof}]
    Let $\tau \in (0,1)$. We begin by showing that \begin{align*}
        Q_{Y|\mX=\bm{x}} (\tau) &= \inf \{ y: F_{Y|\mX=\bm{x}}(y) \geq \tau \} \\
         &=\inf\{ y: \text{P}(Y \leq y |\mX=\bm{x} ) \geq \tau \} \\
        &=\inf\{ y: \text{P}(\mu(\bm{x}) + \sigma(\bm{x}) N_Y \leq y) \geq \tau \} \\
         &=\inf\left\{ y: \text{P}\left(N_Y \leq \frac{y-\mu(\bm{x})}{\sigma(\bm{x})}\right) \geq \tau \right\} \\
         &=\mu(\bm{x}) + \sigma(\bm{x})\inf\left\{n_y: \text{P}(N_Y \leq n_y) \geq \tau \right\} \\
         &=\mu(\bm{x}) + \sigma(\bm{x}) Q_{N_Y}(\tau).
    \end{align*}
    It follows that for any $\tau_1 \neq \tau_2 \in(0,1)$ \begin{align*}
      & Q_{Y|\mX=\bm{x}} (\tau_1) - Q_{Y|\mX=\bm{x}} (\tau_2) \\
     = \quad & \sigma(\bm{x})  Q_{N_Y}(\tau_1) - \sigma(\bm{x})  Q_{N_Y}(\tau_2) \\
          \iff & \sigma(\bm{x}) = \dfrac{Q_{Y|\mX=\bm{x}} (\tau_1) - Q_{Y|\mX=\bm{x}} (\tau_2)}{Q_{N_Y}(\tau_1) -   Q_{N_Y}(\tau_2)}.
    \end{align*}
\end{proof} 
For numerical stability, we use well-separated quantiles, $(\tau_1, \tau_2)=(0.05, 0.95)$, to avoid a conditional quantile difference close to zero. In Appendix \ref{app:ablation-bivariate}, we compare the default setting with $(\tau_1,\tau_2)=(0.25,0.75)$ in the bivariate setting and obtain comparable performance in the downstream causal discovery task.

\subsubsection{Gauss-Legendre Scale Estimation} \label{sec:QR-GL}
Unlike the previous approach, this method uses both the estimated conditional mean and the conditional quantile function to recover $\sigma$. Hence, a drawback of this method is that errors in the estimated conditional mean directly propagate to the variance estimate.

Given a fixed $\mX=\bm{x}$ we have that
\begin{align*}
\text{Var}(Y|\mX=\bm{x})
&= \mathbb{E}\!\left[\left(Y-\mathbb{E}[Y|\mX=\bm{x}]\right)^2 |\mX=\bm{x}\right] \\
&= \int_{\mathcal Y}
\left(y-\mathbb{E}[Y|\mX=\bm{x}]\right)^2
\,dF_{Y|\mX=\bm{x}}(y).
\end{align*} 
Using the well-known quantile representation where for $U\thicksim  \text{Unif}(0,1)$ we get
\(
Y\overset{d}{=} Q_{Y |\mX=\bm{x}}(U),
\)
the conditional variance can also be expressed as
$$
\text{Var}(Y|\mX=\bm{x})
\!=
\int_{0}^1 \!\!
\left(
Q_{Y|\mX=\bm{x}}(u)
-\mathbb{E}[Y|\mX=\bm{x}]
\right)^2 \!du.
$$ Thus, $\sigma(\bf{x})=\sqrt{\text{Var}(Y|\mX=\bm{x})
}$ can be recovered from the conditional quantile function and conditional mean. We approximate the integral using the Gauss-Legendre quadrature~\citep{golub:1969:calculation}, evaluating the estimated quantile function at $m$ nodes. The approximation is exact for polynomials of degree at most $2m-1$ given that the quantiles are correctly estimated. Per default, we set $m=10$. We further evaluate setting $m=100$ in Appendix \ref{app:ablation-bivariate} to assess the sensitivity with respect to the number of knots.

\subsubsection{Suitability of TabPFN} \label{sec:tabpfn-suitability}

To estimate the quantile levels in practice, we perform in-context quantile regression with TabPFN~\citep{hollmann:23:tabpfn}. It is a tabular foundation model pretrained on a large corpus of tabular prediction tasks including classification, regression, and time-series forecasting. Further, it has been adapted to a range of tasks including causal inference~\citep{robertson:26:do-pfn}. We use the publicly available TabPFN-3 model~\citep{grinsztajn:26:tabpfn-v3} which outperforms classical approaches to quantile regression such as quantile XGBoost~\citep{chen:16:xgboost} and quantile regression forests~\citep{meinshausen:06:quantile-reg-forests} on a tailored quantile regression variant of TabArena~\citep{erickson:26:tabarena}. TabPFN-3 decodes quantiles at inference time by inverting the predicted CDF. We use the model as a plug-in estimator without further fine-tuning to obtain the conditional mean and a prespecified set of conditional quantile levels. 

In particular, let $f_\theta$ denote the pretrained TabPFN model and let the context dataset consist of the given iid sample of size $n$, i.e., $\mathcal{D} = \{(\bm{x}_{i}, y_{i})\}_{i=1}^n$. Then we estimate the conditional $\tau$-quantile for the $i$-th datapoint as the in-context prediction at the corresponding training point
\[
\hat{y}_i^{(\tau)} = f_\theta(\mathcal{D}, \bm{x}_i, \tau).
\]
Note that we are only interested in estimating the quantiles within the provided context and thus get the quantiles of the training (context) points. For future work, one could also consider the effect of sample splitting to estimate these quantiles.

While our experiments are based on TabPFN-3, future or commercial models could be used, as well. We also evaluated TabPFN-3.5~\citep{jager:26:tabpfn-v35} which performed, however, subpar to version 3 (cf.~Appendix~\ref{app:ablation-v35}). Further details regarding inference and used model checkpoints are provided in Appendix~\ref{app:hyperparams-loscar}.

\subsection{Consistent Estimation via Feasible GLS} \label{sec:FGLS}
In this section, we discuss the estimation of mean and variance using a spline-based feasible generalized least squares (FGLS) procedure. We detail estimation based on $n$ iid samples $\{(\bm{x}_{i}, y_{i})\}_{i=1}^n$ that are assumed to be generated by the LSNM
\begin{align} \label{eq:fgls-mod}
Y_i=B_1(\mX_i)^\top\bm{\beta} + \exp\left\{\!\frac{\left(\gamma_0+B_2(\mX_i)^\top\bm{\gamma}\right)}{2} \!\right\}N_{Y_i},    
\end{align}
where $N_{Y_i} \Indep \mX_i ,\, \mathbb E[N_{Y_i}]=0$ and $ \mathbb E[N_{Y_i}^2]=1$.
Here, $B_1(\mX_i)$ and $B_2(\mX_i)$ are spline basis functions with fixed dimensions $k_1, k_2$, respectively, that increase model flexibility compared to when using $\mX$ in their place. Unlike $B_1$, we take $B_2$ to exclude an intercept to denote $\gamma_0 \in \mathbb{R}$ separately. 
The conditional mean and variance are given by
\begin{align}
    \mu(\mX_i) &= \E[Y_i | \mX_i]=B_1(\mX_i)^\top\bm{\beta}, \label{eq:fgls-mu} \\ 
   \sigma^2(\mX_i)&= \text{Var}(Y_i| \mX_i)=\exp\left(
 \gamma_0+B_2(\mX_i)^\top\bm{\gamma}\right). \label{eq:fgls-var}
\end{align}

Let $U_i:=\sigma(\mX_i)N_{Y_i}$, $\bm{\Omega}(\bm{\gamma}):=\text{diag}(\exp\{B_2(\mX_i)^\top \bm{\gamma}\})$. The FGLS procedure which produces consistent estimators~\citep[Ch.~8.4]{wooldridge2013introductory} $\widehat{\bm{\beta}}_{\text{FGLS}}$ and $\widehat{\bm{\gamma}}$ under the model in Eq.~\ref{eq:fgls-mod} is given as follows. \begin{enumerate} \setlength{\itemsep}{0.2pt}
    \item Obtain the ordinary least squares (OLS) estimator $\widehat{\bm{\beta}}_{\text{OLS}}$ from regressing $Y_i = B_1(\mX_i)^\top \bm{\beta} + U_i$.
    \item Set $\widehat{U}_i=Y_i-B_1(\mX_i)^\top \widehat{\bm{\beta}}_{\text{OLS}}$.
    \item Let $g_0 \in \mathbb{R}$ and $E[\nu_i| \mX_i]=0$ and obtain $\widehat{\bm{\gamma}}$ via OLS from $\log(\widehat{U}_i^2) = g_0 + B_2(\bm{X}_i)^\top\bm{\gamma}+\nu_i$.
    \item Let $\bm{B}_1$ be the matrix with $i$-th row $B_1(\mX_i)^\top$, $\bm{Y}=(Y_1,\ldots,Y_n)^\top$ and obtain $$\widehat{\bm{\beta}}_{\text{FGLS}} =\left(
\bm{B}_1^\top\,\widehat{\Omega}(\widehat{\bm{\gamma}})^{-1}\bm{B}_1
\right)^{-1}
\bm{B}_1^\top\,\widehat{\Omega}(\widehat{\bm{\gamma}})^{-1}\bm{Y}.$$ 
\end{enumerate}
The third step is motivated by the fact that
\[
\log(U_i^2)=\gamma_0+B_2(\mX_i)^\top\bm{\gamma}+\log(N_{Y_i}^2).
\] Let $r:=\E[\log(N_{Y_i}^2)]$, $\nu_i := \log(N_{Y_i}^2)- r$, $g_0 :=r+\gamma_0$. Then the exogeneity assumption $E[\nu_i| \mX_i]=0$ holds ensuring consistency of the corresponding OLS regression~\citep[Ch.~11]{greene:2000:econometric}. 
By plugging in $\widehat{\bm{\beta}}_{\text{FGLS}}$ and $\widehat{\bm{\gamma}}$ into Eq.~\ref{eq:fgls-mu} and Eq.~\ref{eq:fgls-var}, mean and variance are consistently estimated, with the latter only up to scale, since $\gamma_0$ is not recovered. It follows that the extracted residuals are consistent up to scale which is sufficient for \ourmethod (cf.~Section~\ref{subsec:residual-indep}). 

We note that for numerical stability, we use a small ridge penalty in the empirical implementation of each regression step to mitigate ill-conditioning of the empirical design matrices~\citep{hastie2020ridge}.

\section{Related Work} \label{sec:related-work}

\paragraph{DAG Identifiability} Identifiability conditions are well established for linear non-Gaussian models~\citep[LiNGAM]{LiNGAM}, causal additive models~\citep[CAM]{CAM}, (nonlinear) ANMs~\citep{hoyer:09:nonlinear}, post nonlinear additive noise models~\citep[PNL]{PNL}, and LSNMs~\citep{LOCI,strobl:22:grci}. Both PNL models and LSNMs include ANMs as a special case. We review LSNMs in Section~\ref{sec:lsnms-intro}.

\paragraph{Quantile-Based Causal Discovery} Methods that approach causal discovery through quantile regression are score-based. QCCD~\citep{tagasovska:20:bqcd} employs quantile regression to derive a minimum description length (MDL) based score in line with other approaches building upon MDL-based approximations~\citep{budhathoki:16:origo,marx:17:slope} of the algorithmic Markov condition~\citep{janzing:10:algomarkov}. Recent work derives a score based on quantile partial effects (QPE) to learn the causal order~\citep{chen:26:qpe}, but their estimation method relies on a range of sensitive hyperparameters such as the assumed QPE model or transformation functions. Here, we use in-context quantile regression for estimating conditional mean and standard deviation based on pretrained tabular foundation models which does not require further tuning.

\paragraph{Causal Discovery with LSNMs} Besides the work discussing identifiability of LSNMs, \citet{skewscore} derive a method for symmetric noise leveraging skewness for identification. The code is, however, not publicly available. Further, \textsc{HOST}~\citep{duong:2023:heteroscedastic} iteratively identifies source nodes assuming Gaussianity. \textsc{ROCHE}~\citep{ROCHE} accounts for heavy-tailed noise by modeling a $t$-distribution, and \textsc{SkewD}~\citep{klippert:26:skew} implicitly models skewness by assuming a skew-normal. We include the latter three as baselines. \textsc{ROCHE}, however, can only be applied to the bivariate setting.

\paragraph{Amortized Causal Discovery} Through simulating data from causal networks during training, amortized causal discovery~\citep{lorch:22:amortized-cd} frames causal discovery as a supervised learning task~\citep{helby:22:supervised-cd}. By embedding causal assumptions into data generation and building upon the predictive capabilities of transformer models, more recent approaches show promising results for DAG learning and data imputation~\citep{thompson:26:arrow,li:26:tabcausal,xu:26:in-context} and learning causal graphs in the presence of confounding and non-additive models such as \textsc{CDFM}~\citep{qiao:26:cdfm} and \textsc{FoundCause}~\citep{blobaum:26:foundcause}. Unlike these models, our method does not require extensive training but can leverage any pretrained tabular foundation model that can perform quantile regression. Further, \ourmethod directly exploits identification results for LSNMs. 

\section{Experiments} \label{sec:experiments}
 We evaluate \ourmethod,
 using quantile-difference scale estimation via TabPFN by default, alongside the Gauss-Legendre variant \textsc{\ourmethod-GL} and the FGLS variant \textsc{\ourmethod-FGLS} against state-of-the-art (SOTA) baselines in bivariate and multivariate settings. To avoid any biases due to scaling, we standardize all datasets. We compare accuracy and area under the decision rate curve (AUDRC) in the bivariate setting and F1 score, order errors, structural Hamming distance (SHD) and structural intervention distance (SID) in the multivariate setting. Further metric details and implementation details of the baselines including a runtime analysis are provided in Appendices~\ref{app:metrics} and \ref{app:baseline-hyperparameters}, respectively. Hyperparameters of \ourmethod are specified in Appendix~\ref{app:hyperparams-loscar}.

 \begin{table}[t]
    \caption{Averaged Results on Bivariate Benchmarks. Top two are in bold and the best result is underlined.} 
    \label{tab:summary-all}
    \begin{center}
        \begin{tabular}{l  r  r}
        \toprule
        Method & {Accuracy (\%)~$\uparrow$} & {AUDRC (\%)~$\uparrow$}\\
        \midrule
        \rowcolor{gray!20} \ourmethod & \underline{\textbf{90.04}} $\pm$ 10.84 & \textbf{94.55} $\pm$ 9.17 \, \\
        \rowcolor{gray!20} \textsc{\ourmethod-GL} & \textbf{89.81} $\pm$ 11.26 & 94.36 $\pm$ 9.81 \, \\
        \rowcolor{gray!20} \textsc{\ourmethod-FGLS} & 74.97 $\pm$ 19.49 & 74.57 $\pm$ 26.44\\
        \textsc{SkewD-LL} & 84.56 $\pm$ 19.95 & 89.58 $\pm$ 16.30\\
        \textsc{CAM} & 75.10 $\pm$ 23.17 & 73.90 $\pm$ 26.44\\
        \textsc{QCCD} & 67.64 $\pm$ 26.44 & 69.92 $\pm$ 33.23\\
        \textsc{ROCHE} & 89.52 $\pm$ 9.92 \, & \underline{\textbf{94.99}} $\pm$ 7.56 \, \\
        \textsc{LOCI-IT} & 84.54 $\pm$ 12.71 & 89.43 $\pm$ 13.63\\
        \midrule
        \textsc{FoundCause} & 86.55 $\pm$ 17.96 & 90.97 $\pm$ 14.76\\
        \textsc{CDFM} & 78.75 $\pm$ 15.21 & 87.53 $\pm$ 15.91\\
        \bottomrule
        \end{tabular}
    \end{center}
\end{table}

\subsection{Bivariate Experiments}
We compare against 
\textsc{LOCI-IT}, an independence-based method assuming Gaussianity, \textsc{ROCHE}, \textsc{QCCD} and the likelihood-based \textsc{SkewD-LL} (consistent with multivariate \textsc{SkewD}), as heteroscedastic baselines. 
We further include \textsc{CAM}~\citep{CAM} as an ANM baseline and \textsc{FoundCause} and \textsc{CDFM} as amortized causal discovery baselines.

Our evaluation comprises 19 established benchmark datasets: ANM, LSNM, and multiplicative noise datasets from \citet{tagasovska:20:bqcd}; the SIM datasets and Tübingen pairs~\citep{mooij:16:pairs}; the Dataverse datasets~\citep{guyon2019cause}; and skewed ANM and LSNM datasets from \citet{klippert:26:skew}. Table~\ref{tab:summary-all} reports average accuracy and AUDRC. Detailed dataset descriptions and benchmark-wise results are provided in Appendix~\ref{app:datasets} and \ref{app:bivtables}.

In terms of accuracy, \ourmethod achieves the best performance ($90.04\%$), closely followed by \textsc{\ourmethod-GL} ($89.81\%$). \textsc{\ourmethod-FGLS} trails substantially ($74.97\%$), outperforming only \textsc{QCCD}. This gap between the quantile and FGLS variants of \ourmethod is further illustrated in Figure~\ref{fig:radar-comparision-loscar-variants} showing the benchmark-wise accuracies. When investigating the individual benchmark results, the FGLS variant has a slight edge to the other variants on near Gaussian and low-noise setting. Among the baselines, \textsc{ROCHE} is most competitive ($89.52\%$), while the remaining methods reach at most $86.55\%$ (\textsc{FoundCause}). AUDRC yields a similar ranking, except that \textsc{ROCHE} achieves the highest score, surpassing \ourmethod by $0.44\%$.

We report amortized approaches separately, as both were trained on ANM and LSNM models. \textsc{FoundCause} also includes skewed noise and performs particularly well on such datasets (see Appendix~\ref{app:bivtables}).

\begin{table}[t]
    \setlength{\tabcolsep}{1.5mm}
    \caption{Results for Light Tunnel. Best amortized and non-amortized models marked (OE=Order Errors).} \label{tab:realworld}
    \begin{center}
        \begin{tabular}{lrrrr}
            \toprule
            Method & F1 $\uparrow$ & OE $\downarrow$ & SHD $\downarrow$ & SID $\downarrow$ \\
            \midrule
            \multicolumn{5}{l}{\textbf{Light Tunnel} ($n=10 \, 000$, $d=20$, $\textit{\#edges}=39$)} \\
            \midrule
            \rowcolor{gray!20} \ourmethod & \textbf{0.430} & \textbf{12} & \textbf{45} & \textbf{49} \\
            \rowcolor{gray!20} \textsc{\ourmethod-GL} & \textbf{0.430} & 13 & \textbf{45} & \textbf{49} \\ 
            \rowcolor{gray!20} \textsc{\ourmethod-FGLS} & 0.382 & 16 & 68 & 52 \\
            \textsc{HOST} & 0.000 & 39 & 89 & 181 \\
            \textsc{SkewD} & NA & NA & NA & NA \\
            CAM & 0.213 & 17 & 59 & 158 \\
            LiNGAM & 0.068 & 25 & 82 & 142 \\
            \textsc{NOTEARS-MLP} & 0.092 & 20 & 59 & 138 \\
            \textsc{RESIT-TabPFN} & 0.015 & 38 & 128 & 170 \\
            \midrule
            \textsc{FoundCause} & 0.687 & \textbf{0} & \textbf{21} & \textbf{23} \\
            \textsc{CDFM} & \textbf{0.727} & 1 & 18 & 26 \\
            \bottomrule
        \end{tabular}
    \end{center}
\end{table}

\subsection{Multivariate Experiments}
We evaluate the \ourmethod variants against SOTA baselines on synthetic and real-world multivariate datasets. We include \textsc{SkewD} and \textsc{HOST} as LSNM-specific baselines, \textsc{CAM}, LiNGAM~\citep{LiNGAM} and \textsc{NOTEARS-MLP}~\citep{zheng:20:notears-mlp} as ANM baselines. We additionally consider RESIT with TabPFN as regressor (\textsc{RESIT-TabPFN}). In Appendix~\ref{app:addexp}, we also provide results for RESIT with OLS which is outperformed by \textsc{RESIT-TabPFN}. \textsc{FoundCause} and \textsc{CDFM} are included as amortized baselines.
\begin{figure*}[ht]
    \centering
    \includegraphics[width=0.85\textwidth, trim=0 0.35cm 0 0.4cm, clip]{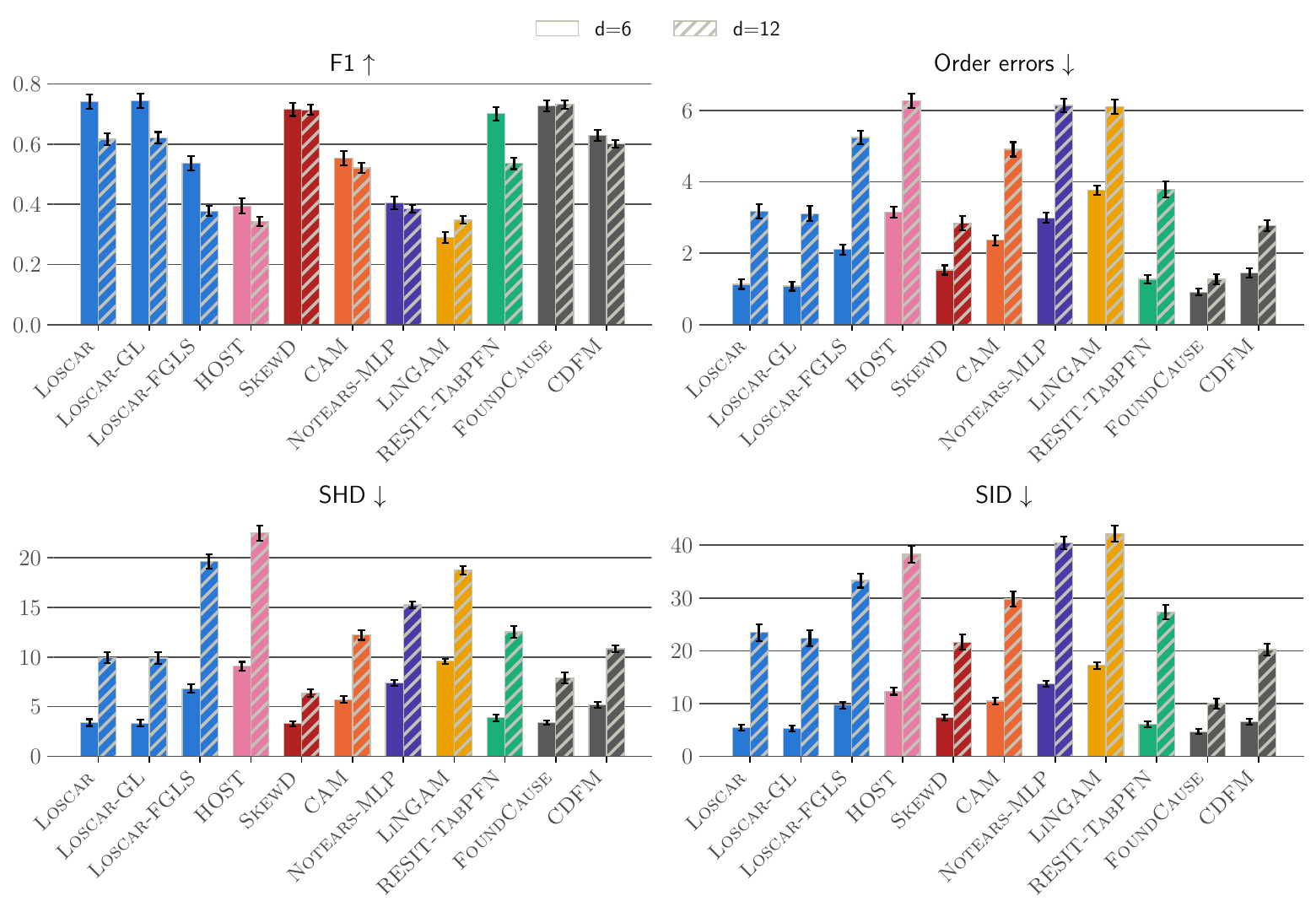}
    \caption{Comparison of \ourmethod and baselines in terms of F1, order errors, SHD and SID on LSNM data.}
    \label{fig:LSs_quad}
\end{figure*}

\paragraph{Synthetic Datasets} We generate Erd{\H{o}}s-R{\'e}nyi graphs~\citep{erdos:1960:evolution} with $d\in\{6,12\}$ nodes and $b=d$ edges. For each $d$, we generate 100 ANM and LSNM datasets each of size $n=1 \, 000$, with $f^{(j)}$ and $g^{(j)}$ given by sums of random sigmoidal functions over the parents of each variable $X^{(j)}$. The noise distribution of each $N_{X^{(j)}}$ is sampled uniformly from Gaussian, exponential, and skew-$t$. Further sampling details and results on homogeneous noise distributions are provided in Appendices~\ref{app:datasets} and \ref{app:homonoise}, respectively.

We focus on the LSNM results in Figure~\ref{fig:LSs_quad}, which show the mean metrics as bar plots including standard errors of the mean. Exact values and corresponding ANM results including precision and recall are reported in Appendices~\ref{app:MVtabMixLSNM} and \ref{app:mixANM}. For $d$ = 6, \ourmethod and \textsc{\ourmethod-GL} achieve the strongest performance among classical baselines in F1, order errors, and SID, while \textsc{SkewD} performs best in SHD. For $d$ = 12, \textsc{SkewD} is generally comparable or slightly better, most notably in F1. Among all methods, \textsc{FoundCause} performs best overall (particularly for $d{=}12$), while \ourmethod and \textsc{\ourmethod-GL} remain competitive. They outperform \textsc{FoundCause} in F1 for $d$ = 6 and, together with \textsc{SkewD}, in SHD for $d$ = 12. \textsc{\ourmethod-FGLS} is not competitive and falls short to \textsc{RESIT-TabPFN}. Overall, the quantile-based variants outperform the classical baselines except \textsc{SkewD}, with \textsc{FoundCause} achieving the strongest overall performance. Notably, in the corresponding ANM experiments (cf. Appendix~\ref{app:mixANM}), \textsc{\ourmethod-FGLS} outperforms the quantile-based variants, whereas \textsc{SkewD} and \textsc{CAM} are the strongest classical baselines.

\paragraph{Real-World Data} For real-world evaluation, we consider the Sachs~\citep{sachs:05:dataset} dataset and Light Tunnel dataset from the Causal Chambers repository~\citep{gamella:25:causal-chambers}, with all details in Appendix~\ref{app:datasets}. Results for Light Tunnel are shown in Table~\ref{tab:realworld}. \ourmethod and \textsc{\ourmethod-GL} are the strongest classical baselines, while the amortized methods substantially outperform all classical approaches across metrics. In particular, \textsc{FoundCause} and \textsc{CDFM} achieve 0 and 1 order errors, respectively, compared with 12 for the second-best being \ourmethod. \textsc{HOST} performs poorly, potentially due to the source nodes being approximately uniform, and \textsc{SkewD} did not terminate after more than 72 hours. Again, note that it is difficult to compare to amortized baselines on synthetic data or Light Tunnel, since the DGPs for these graphs are known. On Sachs most models perform comparably with the amortized baselines outperformed by the classical approaches (cf.~Appendix~\ref{app:sachs}). We treat this result with caution as the ground truth of this dataset has been challenged~\citep{mooij:20:jci}.

\section{Conclusion}
We introduced \ourmethod, a residual-based causal discovery method for location-scale noise models with nonparametric in-context quantile regression using tabular foundation models. We proposed two variants that differ in how the conditional scale is estimated for residual extraction and compared them with a consistent FGLS-based variant. Across our experiments, the quantile-based variants consistently outperform the FGLS variant and achieve state-of-the-art (SOTA) performance on established bivariate benchmarks, while remaining on par with SOTA baselines in low- to medium-dimensional settings.

\paragraph{Limitations and Future Work}
The scalability of \ourmethod is limited by the cost of repeated model fitting and the difficulty of conditional independence testing with large conditioning sets. Moreover, its greedy search may propagate errors across successive steps. To improve robustness to such errors, it would be interesting to extend \ourmethod with beam-search, as explored for RESIT by \citet{ong:2026:causal}, in future work.

\subsection*{Acknowledgements}

This work was partially supported by the Deutsche Forschungsgemeinschaft (DFG, German Research Foundation) -- 598263919. We acknowledge that Quang Nhat Huynh contributed to this work while employed at TU Dortmund University and the Research Center Trustworthy Data Science and Security.

In this work, we used generative AI tools for code prototyping of parts of the code (Claude Sonnet 5 High from Anthropic). Further, we used the ChatGPT 5.6 model from OpenAI for minor
language edits, aiming to enhance readability. LLM-generated code and text was verified, modified and tested for correctness by the authors. 
We take responsibility for the final content of this work, including text, claims, or artifacts produced with the aid of generative AI.

\bibliography{references}

\bibliographystyle{icml2026}

\newpage
\appendix
\onecolumn

The Appendix is structured as follows. Appendix~\ref{app:implemenation} contains all information regarding performance metrics, generated and used datasets and hyperparameter configurations of \ourmethod and the baselines we compare to. Then, Appendix~\ref{app:addexp} contains additional details to the experiments, as well as further ablation studies.

\section{Implementation Details} \label{app:implemenation}

\subsection{Metrics} \label{app:metrics}

We use the following metrics to compare our method to other baselines. 

\paragraph{Bivariate Experiments} For the bivariate experiments, we use \textbf{accuracy}, i.e., the number of correct decisions divided by the total number of pairs and the \textbf{area under the decision rate curve (AUDRC)} which measures how well the confidence of a method is aligned with its correct output~\citep{marx:19:sloper}. For \ourmethod, we compute its confidence as the absolute difference between the HSIC test statistics. To compute the confidence of \textsc{CDFM} and \textsc{FoundCause} we take the absolute difference in prediction probabilities between both directions; all remaining baselines come with their own confidence score. Given accuracy and confidence for each pair, the AUDRC is computed by first sorting all predictions by their confidence and then averaging the accuracies over the top $m$ predictions for $m=1,\ldots,M$, where $M$ is the number of pairs. Denote $\xi(i)$ as the index of the $i$-th most certain prediction, $f(\xi(i))$ the predicted direction and $t(\xi(i))$ the correct direction, then
\begin{equation*}
    \text{AUDRC} = \frac{1}{M} \sum\limits_{m=1}^M \frac{1}{m} \sum\limits_{i=1}^m  \mathds{1}[ f(\xi(i))=t(\xi(i))].
\end{equation*} 
The AUDRC is bounded by 0 and 1, with larger values indicating better performance. Averages over multiple benchmarks are weighted equally, independent of the number of pairs contributing to a benchmark.

For the T{\"u}bingen pairs~\citep{mooij:16:pairs}, we use the recommended weighting to report accuracy and AUDRC.

\paragraph{Multivariate Experiments} To benchmark our approach in the multivariate setting, we compute the following statistics with respect to the ground truth DAG:
\begin{itemize}
    \item \textbf{F1, Precision, Recall:} All metrics compare predicted to ground truth directed edges.
    \item \textbf{Order Errors:} Order errors are computed as suggested by \citet{rolland:22:score}, where given a causal ordering $\pi$ and a target adjacency matrix $A$ (representing a DAG) they define the topological order divergence, here order errors (OE), as
    \begin{equation*} \label{eq:order-errors}
        \mathrm{OE}(\pi, A) = \sum_{i=1}^d \sum_{j : \pi(i) > \pi(j)} A_{ij} \; .
    \end{equation*}
    Here, we have $A_{ij} \in \{ 0, 1 \}$ with $A_{ij}=1$ denoting an existing edge from $i$ to $j$. If $\pi$ is a correct ordering for $A$, $\mathrm{OE}(\pi, A) = 0$. Otherwise, it counts the number of edges that cannot be recovered due to the choice of the causal ordering. Since all methods 
    under comparison learn a DAG, we obtain a causal ordering $\pi$ from their predicted graphs via the function \texttt{topological\_sort} of the package \texttt{NetworkX}~\citep{SciPyProceedings_11}.
    \item \textbf{SHD:} The structural Hamming distance (SHD) measures the number of edges in the predicted adjacency matrix that are either missing or not present in the target graph. Reversed edges are counted as two errors. We use the implementation from the Causal Discovery Toolbox~\citep[CDT]{kalainathan:19:cdt}.
    \item \textbf{SID:} The structural intervention distance~\citep[SID]{peters:2015:structural} measures the closeness of two graphs in terms of their causal inference statements. We use the Python wrapper provided by CDT~\citep{kalainathan:19:cdt}.
\end{itemize}

\subsection{Datasets} \label{app:datasets}

\subsubsection{Generated Synthetic Datasets}
We generate Erd{\H{o}}s-R{\'e}nyi graphs~\citep{erdos:1960:evolution} with $d\in\{6,12\}$ nodes and $b=d$ edges to simulate synthetic data. For each configuration, we construct 100 ANM and LSNM datasets of $n=1\,000$ samples. The functions $f^{(j)}$ and $g^{(j)}$ are generated as sums of random sigmoidal functions of each parent variable of $X^{(j)}$. Following the construction of these sigmoidal functions from \citet{tagasovska:20:bqcd}, we refer to the two structural equation model classes as ANs and LSs, respectively.

For the noise variables $N_{X^{(j)}}$, we consider three distributions with distinct characteristics: symmetric Gaussian noise, positive-valued exponential noise, and unbounded skew-$t$ noise with heavy tails~\citep{skew-normal-book}. In the main experiments, the noise distribution is sampled uniformly from these three choices for each node $N_{X^{(j)}}$. We additionally evaluate each distribution separately and report the corresponding results in Appendix~\ref{app:homonoise}.

For each sample $i$ and node $X^{(j)}$, the noise is generated according to one of the following two-step procedures:
\begin{itemize}
    \item $\sigma_{ij} \thicksim \text{Unif}(1,\sqrt{2}), \quad N_{X^{(j)}_i} \thicksim \mathcal{N}(0,\sigma_{ij}^2)$,
    \item $\lambda_{ij} \thicksim \text{Unif}(1/\sqrt{2}, 1), \quad N_{X^{(j)}_i} \thicksim \text{Exp}(\lambda_{ij})$,
    \item $\omega_{ij} \thicksim \text{Unif}(0.91, 1.29), \quad N_{X^{(j)}_i} \thicksim \text{ST}(0,\omega^2,-2,4)$,
\end{itemize}
where $\text{ST}(\xi, \omega^2, \alpha,\nu)$ denotes the skew-$t$ distribution as defined in \citet{skew-normal-book}. These parameter ranges ensure that the variance of $N_{X^{(j)}_i}$ lies within the interval $(1,2)$ across all three distributions. For non-source nodes $j$, the noise is additionally scaled by $0.2$ in the ANs setting and by $0.35$ in the LSs setting.

\subsubsection{Bivariate Benchmark Datasets}

To assess our methods performance on the bivariate cause-effect discovery task, we evaluate it on benchmarks drawn from four data repositories:
\begin{itemize}
    \item The \textbf{SIM} and \textbf{T{\"u}bingen} datasets were introduced by \citet{mooij:16:pairs}. The \textbf{SIM} datasets each consist of 100 pairs with $n=1 \, 000$ each, with Gaussian-process based functions (\textbf{SIM}), a dataset with low-noise (\textbf{SIM-ln}), near Gaussian distributions (\textbf{SIM-G}) and additional confounding (\textbf{SIM-c}). From the \textbf{T{\"u}bingen} dataset, we extracted all univariate pairs and excluded pairs with binary variables, resulting in 99 pairs. We use the recommended weighting of the pairs to compute the accuracy and AUDRC on this dataset. Note that $n$ varies for this dataset between a few 100 datapoints and more than $30 \, 000$.
    \item \cite{goudet:18:learning} provide the \textbf{Cha} dataset (used in the cause-effect pair challenge), the \textbf{Net} dataset, which is generated from random neural networks, and the \textbf{Multi} dataset generated from polynomial mechanisms. Each dataset consists of 300 cause-effect pairs with $n=1 \, 000$.
    \item \citet{tagasovska:20:bqcd} introduce five synthetic benchmark datasets: two nonlinear additive noise datasets \textbf{AN} and \textbf{ANs}, where the latter uses sigmoidal functions to ensure that the function is invertible and noise distributions drawn from a Gaussian. Accordingly, \textbf{LS} and \textbf{LSs} provide location-scale models, and last \textbf{MNU} models multiplicative noise with uniform noise distributions. All datasets consist of 100 pairs each with $n=1 \, 000$.
    \item \citet{klippert:26:skew} extend the more challenging \textbf{ANs} and \textbf{LSs} datasets from \citet{tagasovska:20:bqcd} to include skewed noise variables. The ``\textbf{-med}'' and ``\textbf{-str}'' variants use noise variables drawn from the skew-normal and a generalized normal distribution with skewness $-0.455$ and $0.985$, respectively. Each of the datasets consists of $200$ cause-effect pairs, equally split between skew-normal and generalized normal. The ``\textbf{-ext}'' datasets consists of 100 pairs with the noise variable drawn from a generalized normal with skewness $1.750$. All pairs consist of $n=1 \, 000$ samples.
\end{itemize}

To avoid any biases due to scaling, we standardize cause and effect by default.

\subsection{Real-World Multivariate}

Our empirical evaluation uses two real-world datasets that we describe below.

\paragraph{Sachs} The Sachs dataset consists of single-cell flow-cytometry measurements of $d=11$ phosphoproteins and phospholipids in human immune cells under several perturbation conditions. The dataset is routinely used to benchmark causal discovery algorithms, whereas its ground truth network has been challenged by multiple authors~\citep{meinshausen:16:methods,mooij:20:jci}. In our comparison, we included the observational-only part of the dataset consisting of $n=853$ data points and $17$ edges commonly used for benchmarking~\citep{cundy:21:bcd-nets}.\!\footnote{\href{https://github.com/ermongroup/BCD-Nets}{https://github.com/ermongroup/BCD-Nets}}

\paragraph{Light Tunnel} The Light Tunnel dataset is part of the Causal Chambers repository~\citep{gamella:25:causal-chambers}. The dataset consists of measurements collected from a controlled physical light-tunnel system, where light-source and polarizer settings affect multiple sensor measurements. We use all $n= 10 \, 000$ observational samples from the uniform-reference condition, providing measurements of the system variables without targeted interventions.\!\footnote{\href{https://github.com/juangamella/causal-chamber/tree/main/datasets/lt\_interventions\_standard\_v1}{https://github.com/juangamella/causal-chamber/tree/main/datasets/lt\_interventions\_standard\_v1}} The ground-truth network consists of $d=20$ nodes and 39 edges.

\subsection{Hyperparameter Configuration \ourmethod} \label{app:hyperparams-loscar}

We implemented our method in Python, extending the implementation of RESIT provided through the \texttt{lingam} package~\citep{ikeuchi:23:python-lingam}\footnote{\href{https://github.com/cdt15/lingam}{https://github.com/cdt15/lingam}} which also provides an efficient implementation of the HSIC independence criterion. For all experiments, we fix the significance threshold $\alpha$ used in Phase 2 of \ourmethod to $0.01$. 

We use TabPFN-3~\citep{grinsztajn:26:tabpfn-v3} with model checkpoint \texttt{tabpfn-v3-regressor-v3\_default.ckpt} as default in our experiments. Conditional mean and quantile levels are obtained through a single fit and predict sequence where the fitting step provides the data as context to the model. Quantile levels are obtained dependent on the instantiation of \ourmethod. \ourmethod estimates the scale based on the quantile difference (cf.~Section~\ref{sec:QR-QD}) with $\tau_1 = 0.05$ and $\tau_2 = 1 - \tau_1$ using the quantiles of the standard normal distribution for the normalization term. \textsc{\ourmethod-GL} selects quantiles to compute the variance integral stated in Section~\ref{sec:QR-GL} using the Gauss-Legendre quadrature \citep{golub:1969:calculation} adjusted to the interval (0,1). This procedure deterministically selects $m$ quantiles at which the quantile function is evaluated to compute the integral. By default we set $m=10$. We ablate both parameter choices in Appendix~\ref{app:ablation-bivariate}. Further, for \textsc{RESIT-TabPFN} we only compute the conditional mean based on TabPFN in the RESIT algorithm.
For both \textsc{\ourmethod-FGLS} and \textsc{RESIT-OLS}, the mean parameters were estimated using ridge regression with a penalty of $\lambda=0.0001$. For each covariate, we used an independent cubic spline basis with 8 features and linear extrapolation beyond the boundary knots. For \textsc{\ourmethod-FGLS}, the variance parameters were estimated using ridge regression with $\lambda=0.01$, with an independent cubic spline basis of 6 features per covariate and the same linear extrapolation scheme.

\textbf{Compute and Runtime:} All experiments were performed single-threaded on an NVIDIA RTX PRO 6000 Blackwell Server Edition GPU node. Variants that did not use TabPFN were deployed on CPU (single-threaded, AMD EPYC 9655 Processor). Running a single pair of the bivariate benchmark took approximately 1 to 3 seconds. For the multivariate synthetic experiments, the runtime for $d=6$ and $d=12$ for a single DAG were $17$ seconds and $\sim 2$ minutes wall-clock time, respectively, for \ourmethod, whereas \textsc{\ourmethod-GL} was faster for $d=12$, with \textsc{\ourmethod-FGLS} being fastest, taking $\sim 3$ seconds for $d=12$. Details for the multivariate synthetic datasets are provided in Appendix~\ref{app:runtime}. As the real-world datasets contained more edges (especially Light Tunnel, which contained 20 nodes, 39 edges, and $10 \, 000$ data points) \ourmethod took longest to compute on these, taking $\sim 14$ minutes on Sachs and $~\sim 85$ minutes on Light Tunnel, respectively. \textsc{\ourmethod-GL} terminated after $\sim 110$ minutes on the Light Tunnel dataset.

\subsection{Baselines} \label{app:baseline-hyperparameters}

We used \textsc{QCCD}~\citep{tagasovska:20:bqcd}, \textsc{ROCHE}~\citep{ROCHE}, and \textsc{LOCI}~\citep{LOCI} with their default parameters as bivariate-only baselines, where for \textsc{LOCI} we selected the better-performing independence based variant. The default methods as published in the corresponding papers were used. Although \textsc{QCCD} is also implemented for the multivariate case, we conducted the multivariate analysis with the stronger-performing baselines. \textsc{CAM}~\citep{CAM} was run with default parameters form the CDT wrapper~\citep{kalainathan:19:cdt} and pruning enabled since we found that activating the pruning step led to significantly stronger performance. LiNGAM was applied with its default parameters as implemented in CDT. \textsc{SkewD} was applied with its default parameters, where we chose the likelihood-based variant for the bivariate comparison to be consistent with the multivariate extension that is also based on this implementation.  For the multivariate version we omitted the optimization step involving the expectation conditional maximization algorithm as advised by the authors to improve scalability. \textsc{NOTEARS-MLP}~\citep{zheng:20:notears-mlp} is applied with its recommended parameters $\lambda_1 = \lambda_2 = 0.01$.

\textsc{FoundCause}~\citep{blobaum:26:foundcause} we use the latest version, v1.0 of \textsc{FoundCause}\footnote{\href{https://github.com/amazon-science/foundcause}{https://github.com/amazon-science/foundcause}} with its default parameters. To force decisions for pairwise causal discovery, we directly compare prediction probabilities for both possible edges $X \to Y$ and $Y \to X$ and use their absolute difference as a score. For \textsc{CDFM}~\citep{qiao:26:cdfm} we apply the same rule to force decisions between pairs and otherwise use the default parameters.\!\footnote{Commit \texttt{9968160}, \href{https://github.com/DMIRLAB-Group/CDFM}{https://github.com/DMIRLAB-Group/CDFM}}

\textbf{Compute and Runtime:} \textsc{FoundCause}, \textsc{CDFM}, and \textsc{HOST} have been evaluated on an NVIDIA RTX PRO 6000 Blackwell Server Edition GPU node. All remaining baselines were executed single-threaded on CPU (AMD EPYC 9655 Processor). We note that there was no computational bottleneck, except for \textsc{SkewD}, for which we terminated the experiment on the Light Tunnel dataset after more than 3 days (2 repeated attempts), however, no GPU compliant version of the method is available. Wall-clock runtimes for each method on the synthetic multivariate experiments are reported in Table~\ref{tab:loscar_runtime} and further visualized in Figure~\ref{fig:runtime}.

\section{Additional Experiments} \label{app:addexp}

\subsection{Additional Bivariate Experiments and Results} \label{app:exp-biv}

\subsubsection{Detailed Bivariate Experimental Results} \label{app:bivtables}
We report the benchmark-wise accuracy and AUDRC obtained by each method in the bivariate setting in Tables~\ref{tab:app-bench-acc} and \ref{tab:bench-audrc}. The quantile-based variants of \ourmethod perform well across all datasets, whereas \textsc{\ourmethod-FGLS} performs well for additive noise but has substantial performance drops in the LSNM case, indicating inappropriate variance estimates. 
Interestingly, increasing the skewness in the LSNM datasets (LSs-med, LSs-str, LSs-ext) leads to a strong increase of performance for \textsc{\ourmethod-FGLS} (37.5\%, 58.0\%, 72.0\%). This behavior is in contrast to what is observed in most other methods. For \textsc{CAM}, \textsc{QCCD} and \textsc{CDFM} performance substantially degrades with increasing skewness, whereas other methods remain stable.

Among the amortized baselines, there are substantial performance differences between \textsc{FoundCause} and \textsc{CDFM} in certain settings. While \textsc{FoundCause} achieves almost perfect accuracy for multiplicative uniform noise (MNU), it seems that \textsc{CDFM} resorts to random guessing in that case. In contrast, when there is integrated confounding (SIMc), \textsc{CDFM} achieves 68\% accuracy, whereas \textsc{FoundCause} achieves 47\%. This suggests certain differences in the training datasets used to obtain these causal foundation models.

\begin{table*}[h!]
    \setlength{\tabcolsep}{1.5mm}
    \caption{Accuracy in \% over All Bivariate Benchmark Datasets. Baselines are split into conventional and amortized. The best result per benchmark is bolded.} \label{tab:app-bench-acc}
    \begin{center}
        \resizebox{\textwidth}{!}{%
        \begin{tabular}{lrrrrrrrrrrrrrrrrrrr}
        \toprule
         & \multicolumn{5}{c}{\citet{tagasovska:20:bqcd}} & \multicolumn{5}{c}{\citet{mooij:16:pairs}} & \multicolumn{3}{c}{\citet{guyon2019cause}} & \multicolumn{6}{c}{\citet{klippert:26:skew}} \\
        \cmidrule(lr){2-6} \cmidrule(lr){7-11} \cmidrule(lr){12-14} \cmidrule(lr){15-20}
        Method & AN & ANs & LS & LSs & MNU & SIM & SIMc & SIMln & SIMG & T{\"u}bingen & Cha & Net & Multi & ANs-med & ANs-str & ANs-ext & LSs-med & LSs-str & LSs-ext \\
        \midrule
        \rowcolor{gray!20} \textsc{\ourmethod} & \textbf{100} & \textbf{100} & \textbf{100} & 99 & \textbf{100} & \textbf{79} & \textbf{82} & 72 & 76 & 71.69 & 75.67 & 85.67 & 85.67 & 99.0 & 98.0 & \textbf{100} & 96.5 & 97.5 & 93 \\
        \rowcolor{gray!20} \textsc{\ourmethod-GL} & \textbf{100} & \textbf{100} & 99 & \textbf{100} & \textbf{100} & 77 & 81 & 73 & 76 & 69.34 & 74.67 & 85.33 & 86.00 & 99.0 & 98.0 & \textbf{100} & 98.0 & 97.0 & 93 \\
        \rowcolor{gray!20} \textsc{\ourmethod-FGLS} & 95 & \textbf{100} & 54 & 38 & 50 & \textbf{79} & \textbf{82} & 81 & 82 & 69.78 & \textbf{76.67} & 81.67 & 77.33 & 97.5 & 94.0 & 99 & 37.5 & 58.0 & 72 \\
        \textsc{SkewD-LL} & \textbf{100} & \textbf{100} & \textbf{100} & \textbf{100} & 87 & 54 & 51 & 81 & 88 & 45.97 & 54.33 & 73.33 & 72.00 & \textbf{100.0} & \textbf{100.0} & \textbf{100} & \textbf{100.0} & \textbf{100.0} & \textbf{100} \\
        \textsc{CAM} & \textbf{100} & \textbf{100} & \textbf{100} & 53 & 86 & 57 & 60 & \textbf{87} & 81 & 57.77 & 46.67 & 78.33 & 34.67 & \textbf{100.0} & \textbf{100.0} & \textbf{100} & 82.5 & 70.0 & 33 \\
        \textsc{ROCHE} & \textbf{100} & \textbf{100} & \textbf{100} & \textbf{100} & \textbf{100} & 78 & \textbf{82} & 77 & 81 & \textbf{77.27} & 74.00 & \textbf{87.67} & 83.00 & 99.5 & 97.5 & 95 & 99.5 & 88.5 & 81 \\
        \textsc{QCCD} & \textbf{100} & 82 & \textbf{100} & 96 & 99 & 62 & 72 & 80 & 64 & 77.09 & 53.67 & 80.33 & 50.67 & 56.0 & 43.5 & 41 & 95.5 & 21.5 & 11 \\
        \textsc{LOCI-IT} & \textbf{100} & \textbf{100} & 94 & 89 & \textbf{100} & 78 & \textbf{82} & 73 & 78 & 60.84 & 72.67 & 86.67 & 79.00 & 99.5 & 99.5 & 98 & 72.0 & 69.0 & 75\\
        \midrule
        \textsc{FoundCause} & \textbf{100} & \textbf{100} & 98 & 94 & 99 & 54 & 47 & 71 & \textbf{93} & 60.95 & 67.00 & 73.67 & \textbf{94.33} & \textbf{100.0} & 99.5 & 93 & \textbf{100.0} & \textbf{100.0} & \textbf{100} \\
        \textsc{CDFM} & 92 & 86 & 92 & 88 & 46 & 52 & 68 & 71 & 88 & 62.30 & 60.33 & 76.67 & 75.00 & 96.5 & 96.5 & 85 & 96.0 & 84.0 & 81 \\
        \bottomrule
    \end{tabular}}
    \end{center}
\end{table*}

\begin{table*}[h!]
    \setlength{\tabcolsep}{1.5mm}
    \caption{AUDRC in \% over All Bivariate Benchmark Datasets. Baselines are split into conventional and amortized. The best result per benchmark is bolded.} \label{tab:bench-audrc}
    \begin{center}
        \resizebox{\textwidth}{!}{%
        \begin{tabular}{lrrrrrrrrrrrrrrrrrrr}
        \toprule
         & \multicolumn{5}{c}{\citet{tagasovska:20:bqcd}} & \multicolumn{5}{c}{\citet{mooij:16:pairs}} & \multicolumn{3}{c}{\citet{guyon2019cause}} & \multicolumn{6}{c}{\citet{klippert:26:skew}} \\
        \cmidrule(lr){2-6} \cmidrule(lr){7-11} \cmidrule(lr){12-14} \cmidrule(lr){15-20}
        Method & AN & ANs & LS & LSs & MNU & SIM & SIMc & SIMln & SIMG & T{\"u}bingen & Cha & Net & Multi & ANs-med & ANs-str & ANs-ext & LSs-med & LSs-str & LSs-ext \\
        \midrule
        \rowcolor{gray!20} \ourmethod & \textbf{100.00} & \textbf{100.00} & \textbf{100.00} & 99.96 & \textbf{100.00} & 87.81 & 93.72 & 87.60 & 94.09 & 70.36 & 71.83 & 96.79 & 95.69 & 99.93 & 99.94 & \textbf{100.00} & 99.79 & 99.85 & 99.09 \\
        \rowcolor{gray!20} \textsc{\ourmethod-GL} & \textbf{100.00} & \textbf{100.00} & 99.98 & \textbf{100.00} & \textbf{100.00} & 88.27 & 92.78 & 87.87 & 94.08 & 66.40 & 71.49 & 96.90 & 96.41 & 99.92 & 99.95 & \textbf{100.00} & 99.93 & 99.86 & 98.92 \\
        \rowcolor{gray!20} \textsc{\ourmethod-FGLS} & 98.32 & \textbf{100.00} & 52.71 & 22.87 & 66.92 & \textbf{88.77} & 92.71 & 84.55 & 93.25 & 62.29 & 76.12 & 92.16 & 77.50 & 99.09 & 97.94 & 99.96 & 16.80 & 40.00 & 54.93 \\
        \textsc{SkewD-LL} & \textbf{100.00} & \textbf{100.00} & \textbf{100.00} & \textbf{100.00} & 97.58 & 64.58 & 64.33 & \textbf{94.57} & 94.95 & 54.57 & 57.50 & 81.93 & 92.01 & \textbf{100.00} & \textbf{100.00} & \textbf{100.00} & \textbf{100.00} & \textbf{100.00} & \textbf{100.00} \\
        \textsc{CAM} & \textbf{100.00} & \textbf{100.00} & \textbf{100.00} & 36.29 & 76.45 & 67.76 & 68.82 & 87.20 & 88.24 & 69.37 & 41.05 & 84.88 & 40.02 & \textbf{100.00} & \textbf{100.00} & \textbf{100.00} & 74.08 & 59.61 & 10.24 \\
        \textsc{ROCHE} & \textbf{100.00} & \textbf{100.00} & \textbf{100.00} & \textbf{100.00} & \textbf{100.00} & 86.61 & \textbf{93.85} & 88.96 & 92.43 & 75.46 & \textbf{78.42} & \textbf{97.62} & 94.40 & \textbf{100.00} & 99.96 & 99.85 & \textbf{100.00} & 99.28 & 97.97 \\
        \textsc{QCCD} & \textbf{100.00} & 91.23 & \textbf{100.00} & 99.87 & 99.99 & 70.69 & 83.31 & 92.19 & 76.30 & \textbf{84.38} & 60.52 & 93.89 & 62.68 & 57.15 & 26.68 & 27.93 & 99.68 & 0.62 & 1.41 \\
        \textsc{LOCI-IT} & \textbf{100.00} & \textbf{100.00} & 99.41 & 97.42 & \textbf{100.00} & 88.74 & 93.16 & 85.97 & 93.24 & 59.62 & 71.33 & 96.97 & 77.76 & 99.99 & 99.99 & 99.89 & 87.28 & 56.98 & 91.39 \\
        \midrule
        \textsc{FoundCause} & \textbf{100.00} & \textbf{100.00} & 99.95 & 99.28 & 99.99 & 63.93 & 56.97 & 80.26 & \textbf{99.42} & 66.81 & 73.33 & 89.23 & \textbf{99.55} & \textbf{100.00} & \textbf{100.00} & 99.68 & \textbf{100.00} & \textbf{100.00} & \textbf{100.00} \\
        \textsc{CDFM} & 95.62 & 97.20 & 98.74 & 98.40 & 38.75 & 64.26 & 76.47 & 86.58 & 95.86 & 75.10 & 71.94 & 90.08 & 89.16 & 99.85 & 99.59 & 97.42 & 99.60 & 95.26 & 93.09 \\
        \bottomrule
        \end{tabular}}
    \end{center}
\end{table*}

\subsubsection{Ablation Studies Bivariate Benchmark} \label{app:ablation-bivariate}

In the following ablation study, we varied several aspects of \ourmethod. For the default variant which uses TabPFN to estimate the conditional mean and estimates the conditional scale based on the quantile difference (Section.~\ref{sec:QR-QD}) we ablate the parameter $\tau_1 \in \{0.05, 0.25 \}$ with $\tau_2=1-\tau_1$, where the default is $0.05$. Similarly, for the Gauss-Legendre variant, we ablate the number of quantiles $m$ that are used for computing the integral. We use $m \in \{ 10, 100 \}$. Further, we added two ANM-based baselines \textsc{RESIT-TabPFN} which implements RESIT using TabPFN as in-context regression method and \textsc{RESIT-OLS} which instead uses spline-based regression.

In Tables~\ref{tab:app-bench-acc-ablation} and \ref{tab:bench-audrc-ablation} we show the accuracy and the AUDRC for each variant on all bivariate benchmarks. We can see that \ourmethod is relatively robust to the parameter choice staying within a $2.5\%$ difference to the main variant for all benchmarks. For \textsc{\ourmethod-GL}, the results are similar, although performance drops of approximately $5\%$ can be observed for T{\"u}bingen and LSs-ext.

To further investigate this behavior, we plot the estimated means and standard deviations for $m=10$ and $m=100$ for Pair 1 of the LSs benchmark as an example in Figure~\ref{fig:ablation-GL10100}. In general, more nodes should improve the integral approximation via the Gauss-Legendre quadrature, as it is exact for polynomials of degree up to $2m-1$. However, note that we do not have access to the ground-truth quantile function, but only an estimate of it. Inaccurate estimates of the quantile function may impair the integral approximation. Figure~\ref{fig:ablation-GL10100} suggests that the quantile function near the quantile levels of 0 and 1 becomes more noisier, as indicated by the more rougher surface of the quantile function when $m=100$, in which case generally more knots are being placed some of which more closely to the limits. In contrast, the most outer quantile function used in the case $m=10$ is much smoother hinting that it might be well-recovered. Overall, the results indicate that a higher $m$ does not result in better performance given that there is no ground-truth quantile function available. Nonetheless, the sensitivity with respect to $m$ is rather small as indicated by the results in Tables~\ref{tab:app-bench-acc-ablation} and \ref{tab:bench-audrc-ablation}.

Overall, \textsc{\ourmethod-FGLS} is clearly outperformed by its TabPFN supported counterparts on data which contains flexible location-scale models, whereas its performance under high skewness improves. Therefore, its mean and variance estimates must be worse than in the TabPFN variants, possibly due to a lack of regularization resulting in too flexible function fits. The ANM baselines are comparable, however, as expected they are outperformed by the location-scale variants of \ourmethod.

\begin{table*}[h!]
    \setlength{\tabcolsep}{1.5mm}
    \caption{Accuracy in \% over All Bivariate Benchmark Datasets. The best result per benchmark is bolded. We include variations of our main approach, ablating $\tau_1$ from its default $0.05$ and $m$ from its default $10$. Further, \textsc{RESIT-TabPFN} is an ANM variant of \ourmethod performing in-context regression with TabPFN and last, \textsc{RESIT-OLS} is another ANM variant of our method leveraging standard spline-based regression.} \label{tab:app-bench-acc-ablation}
    \begin{center}
        \resizebox{\textwidth}{!}{%
        \begin{tabular}{lrrrrrrrrrrrrrrrrrrr}
        \toprule
         & \multicolumn{5}{c}{\citet{tagasovska:20:bqcd}} & \multicolumn{5}{c}{\citet{mooij:16:pairs}} & \multicolumn{3}{c}{\citet{guyon2019cause}} & \multicolumn{6}{c}{\citet{klippert:26:skew}} \\
        \cmidrule(lr){2-6} \cmidrule(lr){7-11} \cmidrule(lr){12-14} \cmidrule(lr){15-20}
        Method & AN & ANs & LS & LSs & MNU & SIM & SIMc & SIMln & SIMG & T{\"u}bingen & Cha & Net & Multi & ANs-med & ANs-str & ANs-ext & LSs-med & LSs-str & LSs-ext \\
        \midrule
        \textsc{\ourmethod} $(\tau=0.05)$ & \textbf{100} & \textbf{100} & \textbf{100} & 99 & \textbf{100} & 79 & \textbf{82} & 72 & 76 & \textbf{71.69} & 75.67 & 85.67 & \textbf{85.67} & 99.0 & 98.0 & \textbf{100} & 96.5 & \textbf{97.5} & \textbf{93} \\
        \ourmethod $(\tau=0.25)$ & \textbf{100} & \textbf{100} & \textbf{100} & 97 & \textbf{100} & 77 & \textbf{82} & 71 & 76 & 69.75 & 72.33 & 85.33 & \textbf{85.67} & 98.5 & 98.0 & \textbf{100} & 96.5 & 96.0 & 92 \\
        \textsc{\ourmethod-GL} ($m=10$) & \textbf{100} & \textbf{100} & 99 & \textbf{100} & \textbf{100} & 77 & 81 & 73 & 76 & 69.34 & 74.67 & 85.33 & 86.00 & 99.0 & 98.0 & \textbf{100} & 98.0 & 97.0 & \textbf{93} \\
        \textsc{\ourmethod-GL} ($m=100$) & \textbf{100} & \textbf{100} & 99 & \textbf{100} & \textbf{100} & \textbf{80} & \textbf{82} & 73 & 77 & 64.41 & 75.00 & \textbf{86.67} & 84.67 & 99.0 & 97.5 & \textbf{100} & \textbf{99.5} & 94.5 & 87 \\
        \textsc{\ourmethod-FGLS} & 95 & \textbf{100} & 54 & 38 & 50 & 79 & \textbf{82} & \textbf{81} & \textbf{82} & 69.78 & \textbf{76.67} & 81.67 & 77.33 & 97.5 & 94.0 & 99 & 37.5 & 58.0 & 72 \\
        \textsc{RESIT-TabPFN} & \textbf{100} & \textbf{100} & 63 & 12 & 0 & 75 & 80 & 74 & 76 & 60.48 & 73.67 & 78.00 & 38.00 & \textbf{100} & \textbf{100} & \textbf{100} & 28.0 & 49.5 & 44 \\
        \textsc{RESIT-OLS} & 98 & \textbf{100} & 65 & 12 & 2 & 74 & \textbf{82} & 78 & 77 & 60.01 & 73.33 & 78.00 & 37.00 & \textbf{100} & \textbf{100} & 99 & 31.5 & 50.5 & 43 \\
        \bottomrule
        \end{tabular}}
    \end{center}
\end{table*}
\begin{table*}[h!]
    \setlength{\tabcolsep}{1.5mm}
    \caption{AUDRC in \% over All Bivariate Benchmark Datasets. The best result per benchmark is bolded. We include variations of our main approach, ablating $\tau_1$ from its default $0.05$ and $m$ from its default $10$. Further, \textsc{RESIT-TabPFN} is an ANM variant of \ourmethod performing in-context regression with TabPFN and last, \textsc{RESIT-OLS} is another ANM variant of our method leveraging standard spline-based regression.} \label{tab:bench-audrc-ablation}
    \begin{center}
        \resizebox{\textwidth}{!}{%
        \begin{tabular}{lrrrrrrrrrrrrrrrrrrr}
        \toprule
         & \multicolumn{5}{c}{\citet{tagasovska:20:bqcd}} & \multicolumn{5}{c}{\citet{mooij:16:pairs}} & \multicolumn{3}{c}{\citet{guyon2019cause}} & \multicolumn{6}{c}{\citet{klippert:26:skew}} \\
        \cmidrule(lr){2-6} \cmidrule(lr){7-11} \cmidrule(lr){12-14} \cmidrule(lr){15-20}
        Method & AN & ANs & LS & LSs & MNU & SIM & SIMc & SIMln & SIMG & T{\"u}bingen & Cha & Net & Multi & ANs-med & ANs-str & ANs-ext & LSs-med & LSs-str & LSs-ext \\
        \midrule
        \textsc{\ourmethod} $(q=0.05)$ & \textbf{100.00} & \textbf{100.00} & \textbf{100.00} & 99.96 & \textbf{100.00} & 87.81 & \textbf{93.72} & 87.60 & \textbf{94.09} & \textbf{70.36} & 71.83 & 96.79 & 95.69 & 99.93 & 99.94 & \textbf{100.00} & 99.79 & 99.85 & \textbf{99.09} \\
        \textsc{\ourmethod} $(q=0.25)$ & \textbf{100.00} & \textbf{100.00} & \textbf{100.00} & 99.94 & \textbf{100.00} & 86.84 & 92.59 & 89.43 & 93.15 & 65.75 & 69.67 & 96.17 & 95.90 & 99.93 & 99.92 & \textbf{100.00} & 99.78 & 99.82 & 98.91 \\
        \textsc{\ourmethod--GL} ($m=10$) & \textbf{100.00} & \textbf{100.00} & 99.98 & \textbf{100.00} & \textbf{100.00} & 88.27 & 92.78 & 87.87 & 94.08 & 66.40 & 71.49 & 96.90 & \textbf{96.41} & 99.92 & 99.95 & \textbf{100.00} & 99.93 & \textbf{99.86} & 98.92 \\
        \textsc{\ourmethod-GL} ($m=100$) & \textbf{100.00} & \textbf{100.00} & 99.98 & \textbf{100.00} & \textbf{100.00} & \textbf{89.13} & 93.28 & 89.07 & 93.89 & 56.81 & 71.56 & \textbf{96.99} & 94.76 & 99.98 & 99.95 & \textbf{100.00} & \textbf{99.99} & 99.21 & 96.55 \\
        \textsc{\ourmethod-FGLS} & 98.32 & \textbf{100.00} & 52.71 & 22.87 & 66.92 & 88.77 & 92.71 & 84.55 & 93.25 & 62.29 & \textbf{76.12} & 92.16 & 77.50 & 99.09 & 97.94 & 99.96 & 16.80 & 40.00 & 54.93 \\
        \textsc{RESIT-TabPFN} & \textbf{100.00} & \textbf{100.00} & 70.87 & 2.63 & 0.00 & 84.47 & 90.25 & \textbf{89.74} & 90.15 & 53.74 & 70.58 & 91.72 & 20.88 & \textbf{100.00} & \textbf{100.00} & \textbf{100.00} & 10.94 & 23.20 & 16.54 \\
        \textsc{RESIT-OLS} & 99.93 & \textbf{100.00} & 77.16 & 2.78 & 0.11 & 87.56 & 91.50 & 87.29 & 90.59 & 54.00 & 72.90 & 90.87 & 22.54 & \textbf{100.00} & \textbf{100.00} & 99.99 & 13.28 & 29.47 & 18.01 \\
        \bottomrule
        \end{tabular}}
    \end{center}
\end{table*}
\begin{figure}[h!]
    \centering
    \includegraphics[width=\linewidth]{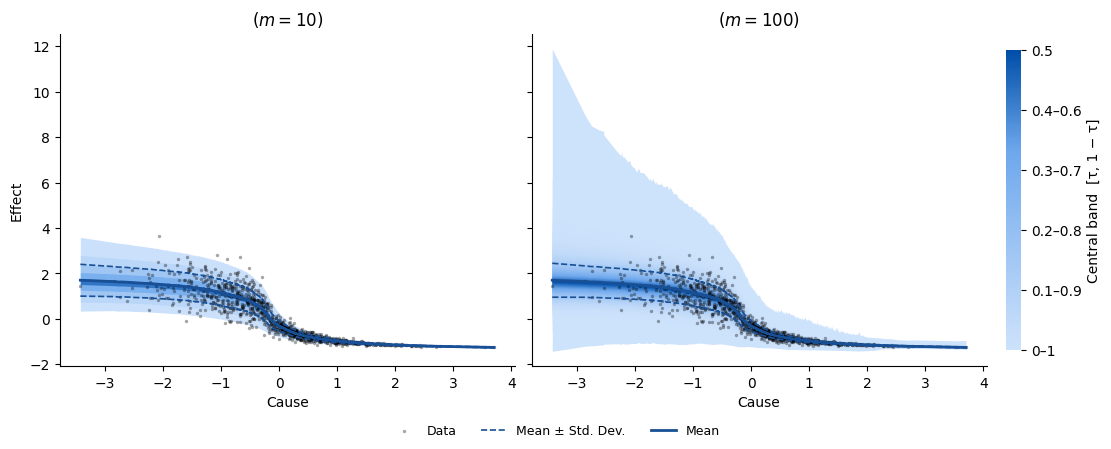}
    \caption{Shown are the estimates of the conditional mean and quantiles obtained from TabPFN-3 with $m \in \{10, 100\}$ and the corresponding scale computed as described in Section~\ref{sec:QR-GL} on Pair~1 of the LSs dataset~\citep{tagasovska:20:bqcd}. The outer contours of the shaded area correspond to the highest and lowest quantile considered by the corresponding instantiation. The estimates are used for \textsc{\ourmethod-GL}, where $m=10$ is the default used in our experiments.}
    \label{fig:ablation-GL10100}
\end{figure}

\subsubsection{Remarks to TabPFN Version 3.5} \label{app:ablation-v35}

In this article, we use TabPFN-3~\citep{grinsztajn:26:tabpfn-v3}, while in theory, our approach is agnostic to the model version of TabPFN as long as this version supports mean and quantile regression. During the time of writing, version 3.5 of TabPFN has been released~\citep{jager:26:tabpfn-v35}. Therefore, we ran a preliminary ablation study to compare both versions, where for v3.5, we use the model checkpoint \texttt{tabpfn-v3.5-20260909.safetensors}, Release v9.0.0. 
We found that in our framework, TabPFN-3 performed superior to TabPFN-3.5 and traced the performance difference to the fact that TabPFN-3 returned smoother estimates of conditional mean and estimated quantiles. We illustrate this observation Figure~\ref{fig:ablation-pfn35}, where we plot the estimated conditional mean, the default quantiles for \ourmethod as well as additional quantile ranges for Pair~1 of the LSs dataset~\citep{tagasovska:20:bqcd}. 
\begin{figure}[h!]
    \centering
    \includegraphics[width=0.99\linewidth]{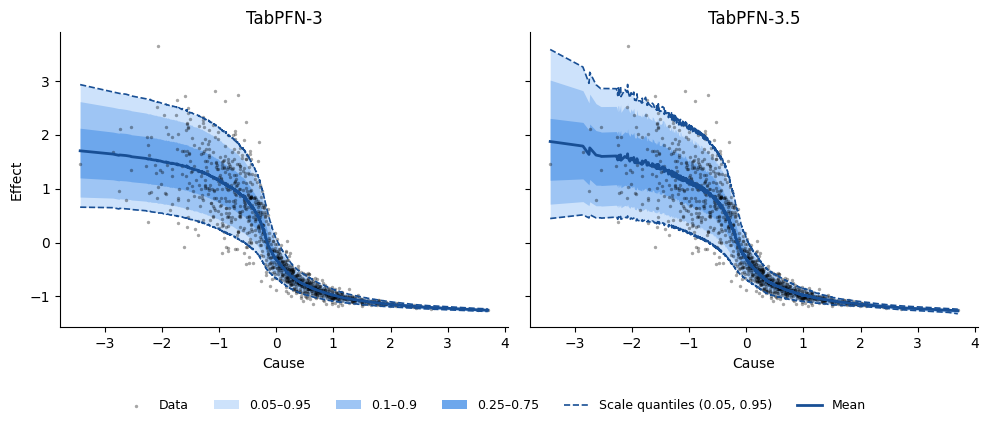}
    \caption{Shown are the estimates of the conditional mean and quantiles obtained from TabPFN-3 (left) and TabPFN-3.5 (right) on Pair~1 of the LSs dataset~\citep{tagasovska:20:bqcd}. It can be observed that the estimates for TabPFN-3 exhibit greater smoothness.}
    \label{fig:ablation-pfn35}
\end{figure}

\subsection{Additional Multivariate Experiments and Results}

\subsubsection{Results on the Sachs Dataset} \label{app:sachs}

We show the results in terms of F1, order errors, SHD and SID for the Sachs dataset in Table~\ref{tab:realworld-sachs}. We see that \textsc{CAM} and \textsc{HOST} perform best, despite their weak performance on Light Tunnel. \ourmethod and \ourmethod-GL perform equally with CAM in terms of order errors and fall short by 2 and 4, respectively, in terms of SHD to HOST. The amortized methods perform comparably to \ourmethod and \textsc{\ourmethod-GL}.

\begin{table}[h!]
    \setlength{\tabcolsep}{1.5mm}
    \caption{Results on F1, Order Errors (OE), SHD and SID for the Sachs Dataset. The best result for amortized and non-amortized models each is bolded.} \label{tab:realworld-sachs}
    \begin{center}
        \begin{tabular}{lrrrr}
            \toprule
            Method & F1 $\uparrow$ & OE $\downarrow$ & SHD $\downarrow$ & SID $\downarrow$ \\
            \midrule
            \multicolumn{5}{l}{\textbf{Sachs} ($n=853$, $d=11$, $\textit{\#edges}=17$)} \\
            \midrule
            \rowcolor{gray!20} \ourmethod & 0.296 & 7 & 19 & 55 \\
            \rowcolor{gray!20} \textsc{\ourmethod-GL} & 0.320 & 7 & 17 & 55 \\
            \rowcolor{gray!20} \textsc{\ourmethod-FGLS} & 0.150 & 9 & 34 & 55 \\
            \textsc{HOST} & 0.400 & \textbf{3} & \textbf{15} & \textbf{47} \\
            \textsc{SkewD} & 0.296 & 7 & 19 & 55 \\
            CAM & \textbf{0.444} & 7 & \textbf{15} & 55 \\
            LiNGAM & 0.320 & 7 & 17 & 55 \\
            \textsc{NOTEARS-MLP} & 0.333 & 6 & 16 & 48 \\
            \textsc{RESIT-TabPFN} & 0.286 & 9 & 25 & 52 \\
            \midrule
            \textsc{FoundCause} & 0.296 & \textbf{8} & 19 & \textbf{50} \\
            \textsc{CDFM} & \textbf{0.307} & 11 & \textbf{18} & \textbf{50} \\
            \bottomrule
        \end{tabular}
    \end{center}
\end{table}

\subsubsection{Tabular Comparison under Mixed Location-Scale Noise} \label{app:MVtabMixLSNM}

We complement the results on synthetic multivariate data with heterogeneous noise distributions shown in Figure~\ref{fig:LSs_quad} in Section~\ref{sec:experiments} by reporting the exact mean performances achieved by the baselines including \textsc{RESIT-OLS} in Table~\ref{tab:loscar_lss}. Precision and recall are additionally included, with each variant of \ourmethod generally obtaining higher precision than recall.

\begin{table*}[h!]
\centering
\caption{Results on Mean F1, Precision, Recall, Order Errors, SHD and SID on Synthetic Multivariate LSNM Data with Heterogeneous Noise. Baselines are split into conventional and
amortized. The best result per benchmark is bolded.}
\label{tab:loscar_lss}
\begin{center}
        \resizebox{\textwidth}{!}{
\begin{tabular}{l r r r r r r r r r r r r}
\toprule
 & \multicolumn{2}{c}{F1 $\uparrow$} & \multicolumn{2}{c}{Precision $\uparrow$} & \multicolumn{2}{c}{Recall $\uparrow$} & \multicolumn{2}{c}{Order errors $\downarrow$} & \multicolumn{2}{c}{SHD $\downarrow$} & \multicolumn{2}{c}{SID $\downarrow$} \\
\cmidrule(lr){2-3} \cmidrule(lr){4-5} \cmidrule(lr){6-7} \cmidrule(lr){8-9} \cmidrule(lr){10-11} \cmidrule(lr){12-13}
Method & d=6 & d=12 & d=6 & d=12 & d=6 & d=12 & d=6 & d=12 & d=6 & d=12 & d=6 & d=12 \\
\midrule
\rowcolor{gray!20} \textsc{Loscar} & \textbf{0.740} & 0.616 & \textbf{0.749} & 0.605 & 0.742 & 0.638 & 1.13 & 3.17 & 3.40 & 9.96 & 5.43 & 23.46 \\
\rowcolor{gray!20} \textsc{Loscar-GL} & \textbf{\underline{0.745}} & 0.620 & \textbf{\underline{0.753}} & 0.608 & \textbf{0.747} & 0.643 & \textbf{1.08} & 3.11 & \textbf{3.33} & 9.88 & \textbf{5.31} & 22.40 \\
\rowcolor{gray!20} \textsc{Loscar-FGLS} & 0.537 & 0.377 & 0.504 & 0.328 & 0.595 & 0.463 & 2.10 & 5.24 & 6.82 & 19.57 & 9.73 & 33.32 \\
\textsc{HOST} & 0.394 & 0.343 & 0.353 & 0.282 & 0.458 & 0.457 & 3.15 & 6.27 & 9.07 & 22.46 & 12.33 & 38.27 \\
\textsc{SkewD} & 0.715 & \textbf{0.714} & 0.748 & \textbf{\underline{0.762}} & 0.690 & \textbf{0.675} & 1.53 & 2.84 & \textbf{\underline{3.29}} & \textbf{\underline{6.37}} & 7.36 & 21.61 \\
\textsc{CAM} & 0.553 & 0.521 & 0.543 & 0.504 & 0.572 & 0.544 & 2.36 & 4.91 & 5.71 & 12.19 & 10.48 & 29.82 \\
\textsc{NOTEARS-MLP} & 0.404 & 0.384 & 0.400 & 0.376 & 0.420 & 0.399 & 2.99 & 6.14 & 7.39 & 15.24 & 13.78 & 40.41 \\
LiNGAM & 0.289 & 0.347 & 0.264 & 0.306 & 0.332 & 0.415 & 3.77 & 6.10 & 9.56 & 18.74 & 17.24 & 42.25 \\
\textsc{RESIT-TabPFN} & 0.701 & 0.535 & 0.703 & 0.510 & 0.710 & 0.573 & 1.27 & 3.78 & 3.88 & 12.52 & 6.11 & 27.34 \\
\textsc{RESIT-OLS} & 0.286 & 0.204 & 0.218 & 0.135 & 0.432 & 0.443 & 3.28 & 6.31 & 13.22 & 43.53 & 12.68 & 34.46 \\ \midrule
\textsc{FoundCause} & 0.728 & \textbf{\underline{0.732}} & 0.739 & \textbf{0.681} & \textbf{\underline{0.752}} & \textbf{\underline{0.829}} & \textbf{\underline{0.91}} & \textbf{\underline{1.27}} & 3.40 & \textbf{7.89} & \textbf{\underline{4.70}} & \textbf{\underline{10.01}} \\
\textsc{CDFM} & 0.629 & 0.600 & 0.587 & 0.554 & 0.703 & 0.673 & 1.45 & \textbf{2.77} & 5.20 & 10.82 & 6.56 & \textbf{20.24} \\
\bottomrule
\end{tabular}}
    \end{center}
\end{table*}

\subsubsection{Comparison under Mixed Additive Noise}\label{app:mixANM}

We visualize the results achieved by all baselines across the main metrics for the multivariate synthetic ANM datasets with heterogeneous noise via the bar plot given by Figure~\ref{fig:ANs_quad} showcasing both the mean and its standard error. Exact performances are supplied in Table~\ref{tab:loscar_ans} which also displays precision and recall. The strongest conventional baselines across all metrics are \textsc{SkewD} and \textsc{CAM} with the latter achieving better results. As an example, in terms of F1 \textsc{CAM} achieves 0.839 when $d=6$, whereas \ourmethod, \textsc{\ourmethod-GL} and \textsc{\ourmethod-FGLS} achieve 0.612, 0.632 and 0.718, respectively. Furthermore, \textsc{SkewD} and \textsc{CAM} outperform both amortized baselines in terms of F1 and SHD, whereas in terms of order errors and SID either \textsc{FoundCause} or \textsc{CDFM} yields the best result. 

Unlike in the LSNM case, \textsc{\ourmethod-FGLS} outperforms its quantile-based counterparts. Moreover, as in the LSNM case, the performances of all variants of \ourmethod drop notably when the dimensionality increases from $d=6$ to $d=12$. While this behavior is observed for most methods, there are some exceptions. Most notably, \textsc{CAM} is barely affected in terms of F1 and \textsc{FoundCause} improves when increasing $d$ in terms of Recall (0.797 to 0.903) and order errors (0.83 to 0.74) showcasing that amortized baselines can behave unexpectedly.

\begin{figure}[h!]
    \centering
    \includegraphics[width=0.99\textwidth]{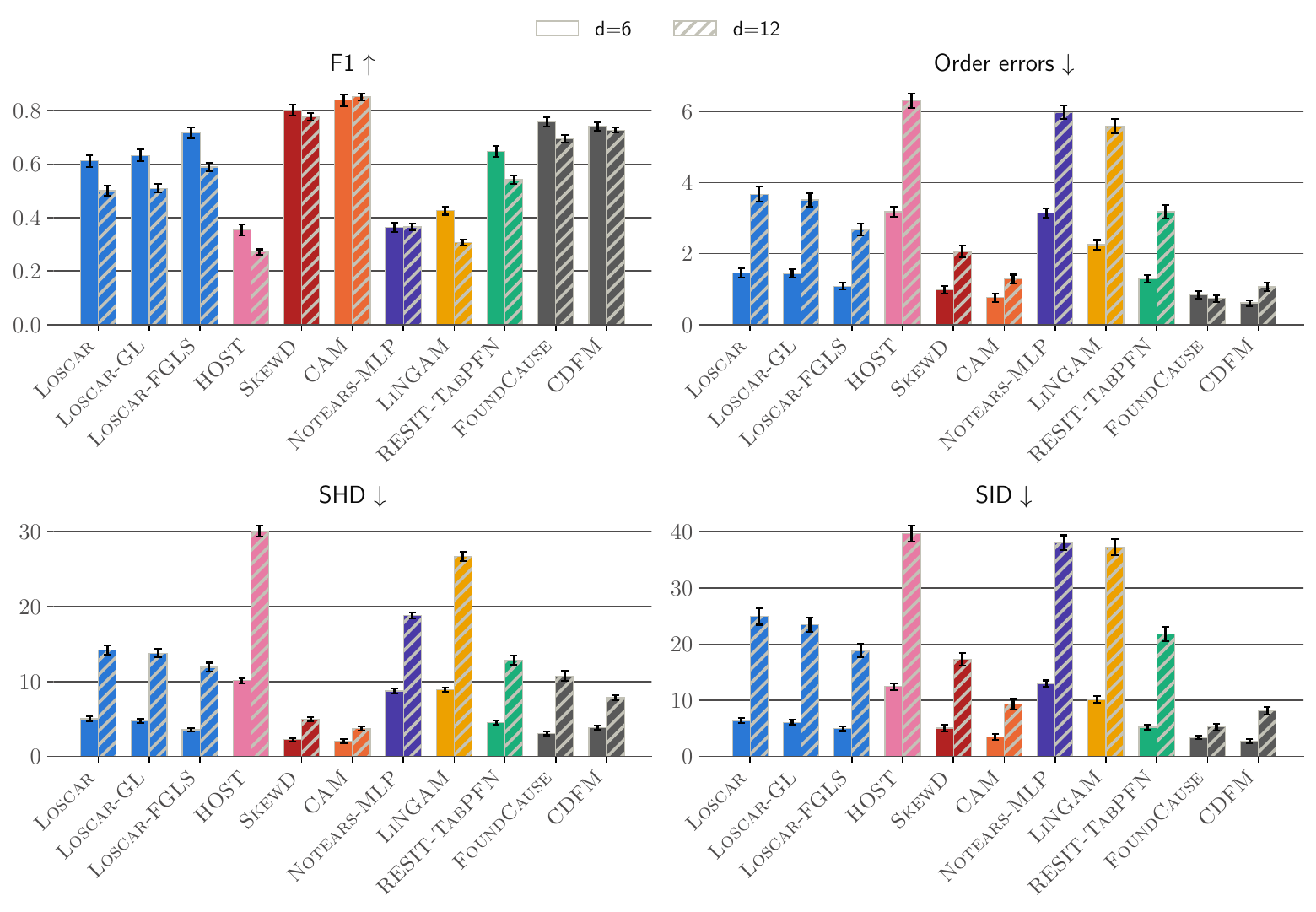}
    \caption{Comparison of \ourmethod and baselines in terms of F1, order errors, SHD and SID on ANM data.}
    \label{fig:ANs_quad}
\end{figure}

\begin{table*}[h!]
\centering
\caption{Results on Mean F1, Precision, Recall, Order Errors, SHD and SID on Synthetic Multivariate ANM Data with Heterogeneous Noise.}
\label{tab:loscar_ans}
\begin{center}
        \resizebox{\textwidth}{!}{
\begin{tabular}{l r r r r r r r r r r r r}
\toprule
 & \multicolumn{2}{c}{F1 $\uparrow$} & \multicolumn{2}{c}{Precision $\uparrow$} & \multicolumn{2}{c}{Recall $\uparrow$} & \multicolumn{2}{c}{Order errors $\downarrow$} & \multicolumn{2}{c}{SHD $\downarrow$} & \multicolumn{2}{c}{SID $\downarrow$} \\
\cmidrule(lr){2-3} \cmidrule(lr){4-5} \cmidrule(lr){6-7} \cmidrule(lr){8-9} \cmidrule(lr){10-11} \cmidrule(lr){12-13}
Method & d=6 & d=12 & d=6 & d=12 & d=6 & d=12 & d=6 & d=12 & d=6 & d=12 & d=6 & d=12 \\
\midrule
\rowcolor{gray!20} \textsc{Loscar} & 0.612 & 0.501 & 0.614 & 0.460 & 0.625 & 0.560 & 1.45 & 3.67 & 5.02 & 14.19 & 6.38 & 24.90 \\
\rowcolor{gray!20} \textsc{Loscar-GL} & 0.632 & 0.509 & 0.632 & 0.467 & 0.647 & 0.570 & 1.44 & 3.51 & 4.75 & 13.78 & 6.04 & 23.43 \\
\rowcolor{gray!20} \textsc{Loscar-FGLS} & 0.718 & 0.589 & 0.716 & 0.539 & 0.730 & 0.667 & 1.08 & 2.68 & 3.56 & 11.89 & 4.93 & 18.84 \\
\textsc{HOST} & 0.354 & 0.271 & 0.303 & 0.196 & 0.442 & 0.454 & 3.18 & 6.31 & 10.10 & 30.09 & 12.44 & 39.67 \\
\textsc{SkewD} & \textbf{0.802} & \textbf{0.777} & \textbf{\underline{0.840}} & \textbf{0.826} & 0.773 & 0.738 & 0.98 & 2.06 & \textbf{2.23} & \textbf{5.00} & 5.06 & 17.30 \\
\textsc{CAM} & \textbf{\underline{0.839}} & \textbf{\underline{0.852}} & \textbf{0.838} & \textbf{\underline{0.840}} & \textbf{0.847} & \textbf{0.869} & \textbf{0.76} & 1.28 & \textbf{\underline{2.05}} & \textbf{\underline{3.72}} & 3.45 & 9.30 \\
\textsc{NOTEARS-MLP} & 0.363 & 0.365 & 0.336 & 0.312 & 0.407 & 0.448 & 3.15 & 5.98 & 8.73 & 18.80 & 12.99 & 38.01 \\
LiNGAM & 0.426 & 0.307 & 0.353 & 0.229 & 0.560 & 0.482 & 2.25 & 5.59 & 8.93 & 26.66 & 10.14 & 37.27 \\
\textsc{RESIT-TabPFN} & 0.648 & 0.542 & 0.646 & 0.504 & 0.662 & 0.598 & 1.29 & 3.18 & 4.51 & 12.83 & 5.18 & 21.76 \\
\textsc{RESIT-OLS} & 0.648 & 0.527 & 0.597 & 0.435 & 0.738 & 0.711 & 1.29 & 2.87 & 5.38 & 16.94 & 5.79 & 19.39 \\        \midrule
\textsc{FoundCause} & 0.758 & 0.695 & 0.751 & 0.595 & 0.797 & \textbf{\underline{0.903}} & 0.83 & \textbf{\underline{0.74}} & 3.05 & 10.71 & \textbf{3.39} & \textbf{\underline{5.24}} \\
\textsc{CDFM} & 0.742 & 0.727 & 0.667 & 0.657 & \textbf{\underline{0.862}} & 0.838 & \textbf{\underline{0.61}} & \textbf{1.06} & 3.84 & 7.83 & \textbf{\underline{2.68}} & \textbf{8.12} \\
\bottomrule
\end{tabular}}
    \end{center}
\end{table*}

\subsubsection{Runtime Comparison under Mixed Noise Distribution} \label{app:runtime}

In Table~\ref{tab:loscar_runtime} and Figure~\ref{fig:runtime}, we show the average runtime for our synthetic data experiments with mixed noise distribution per mechanism type and number of nodes $d$. We can see that all approaches terminate within a reasonable time. While \textsc{SkewD} has the longest runtime (on CPU), taking on average less than 4 minutes for ANs with $d=12$, \ourmethod has a comparable runtime with $\sim 2$ minutes for the $d=12$ datasets, whereas \textsc{\ourmethod-GL} is faster and takes at most $\sim 75$ seconds for ANs with $d=12$. We note, however, that our method scales better than \textsc{SkewD} for larger datasets. \textsc{SkewD} did not terminate on Light Tunnel ($d=20$ with 39 edges) within 72 hours, whereas \ourmethod and \textsc{\ourmethod-GL} took around 85 and 110 minutes, respectively. All amortized methods finish below one second and the remaining conventional baselines, which we outperform in terms of performance, are also faster. HOST, the other heteroscedastic baseline, however, also takes 51 seconds on LSs ($d=12$).

\begin{table}[ht]
\caption{Mean $\pm$ Standard Deviation Runtime (Seconds) based on Synthetic Multivariate Data with Heterogeneous Noise. Baselines are split into conventional and
amortized. The best result per benchmark is bolded.}
\label{tab:loscar_runtime}
\begin{center}
\begin{tabular}{l r r r r}
\toprule
 & \multicolumn{2}{c}{LSs} & \multicolumn{2}{c}{ANs} \\
\cmidrule(lr){2-3} \cmidrule(lr){4-5}
Method & d=6 & d=12 & d=6 & d=12 \\
\midrule
\rowcolor{gray!20} \textsc{\ourmethod} & 17.90 $\pm$ 1.97 & 127.04 $\pm$ 21.30 & 17.61 $\pm$ 1.73 & 125.59 $\pm$ 22.04 \\
\rowcolor{gray!20} \textsc{\ourmethod-GL} & 32.42 $\pm$ 1.46 & 72.99 $\pm$ 4.41 & 31.41 $\pm$ 1.59 & 75.55 $\pm$ 3.48 \\
\rowcolor{gray!20} \textsc{\ourmethod-FGLS} & 0.58 $\pm$ 0.01 & 2.83 $\pm$ 0.03 & 0.58 $\pm$ 0.01 & 2.82 $\pm$ 0.03 \\
\textsc{HOST} & 17.70 $\pm$ 7.28 & 51.57 $\pm$ 35.48 & 13.04 $\pm$ 8.28 & 22.50 $\pm$ 1.13 \\
\textsc{SkewD} & 50.38 $\pm$ 25.91 & 165.84 $\pm$ 60.74 & 49.36 $\pm$ 23.37 & 215.87 $\pm$ 88.02 \\
\textsc{CAM} & 1.27 $\pm$ 0.13 & 2.99 $\pm$ 0.59 & 1.24 $\pm$ 0.12 & 2.61 $\pm$ 0.49 \\
\textsc{NOTEARS-MLP} & 2.78 $\pm$ 0.81 & 6.40 $\pm$ 1.32 & 3.68 $\pm$ 1.28 & 12.33 $\pm$ 7.07 \\
LiNGAM & 0.72 $\pm$ 0.10 & 0.83 $\pm$ 0.14 & 0.61 $\pm$ 0.12 & 0.68 $\pm$ 0.09 \\
\textsc{RESIT-TabPFN} & 18.87 $\pm$ 0.64 & 64.26 $\pm$ 2.72 & 19.02 $\pm$ 0.67 & 69.66 $\pm$ 8.44 \\
\textsc{RESIT-OLS} & 0.52 $\pm$ 0.01 & 2.32 $\pm$ 0.03 & 0.51 $\pm$ 0.00 & 2.30 $\pm$ 0.02 \\ \midrule
\textsc{FoundCause} & \textbf{0.36 $\pm$ 0.02} & \textbf{0.53 $\pm$ 0.04} & \textbf{0.37 $\pm$ 0.04} & \textbf{0.53 $\pm$ 0.04} \\
\textsc{CDFM} & \textbf{\underline{0.02}} $\pm$ 0.01 & \textbf{\underline{0.03}} $\pm$ 0.01 & \textbf{\underline{0.02}} $\pm$ 0.01 & \textbf{\underline{0.03}} $\pm$ 0.01 \\
\bottomrule
\end{tabular}
\end{center}
\end{table}

\begin{figure}
    \centering
    \includegraphics[width=0.99\linewidth]{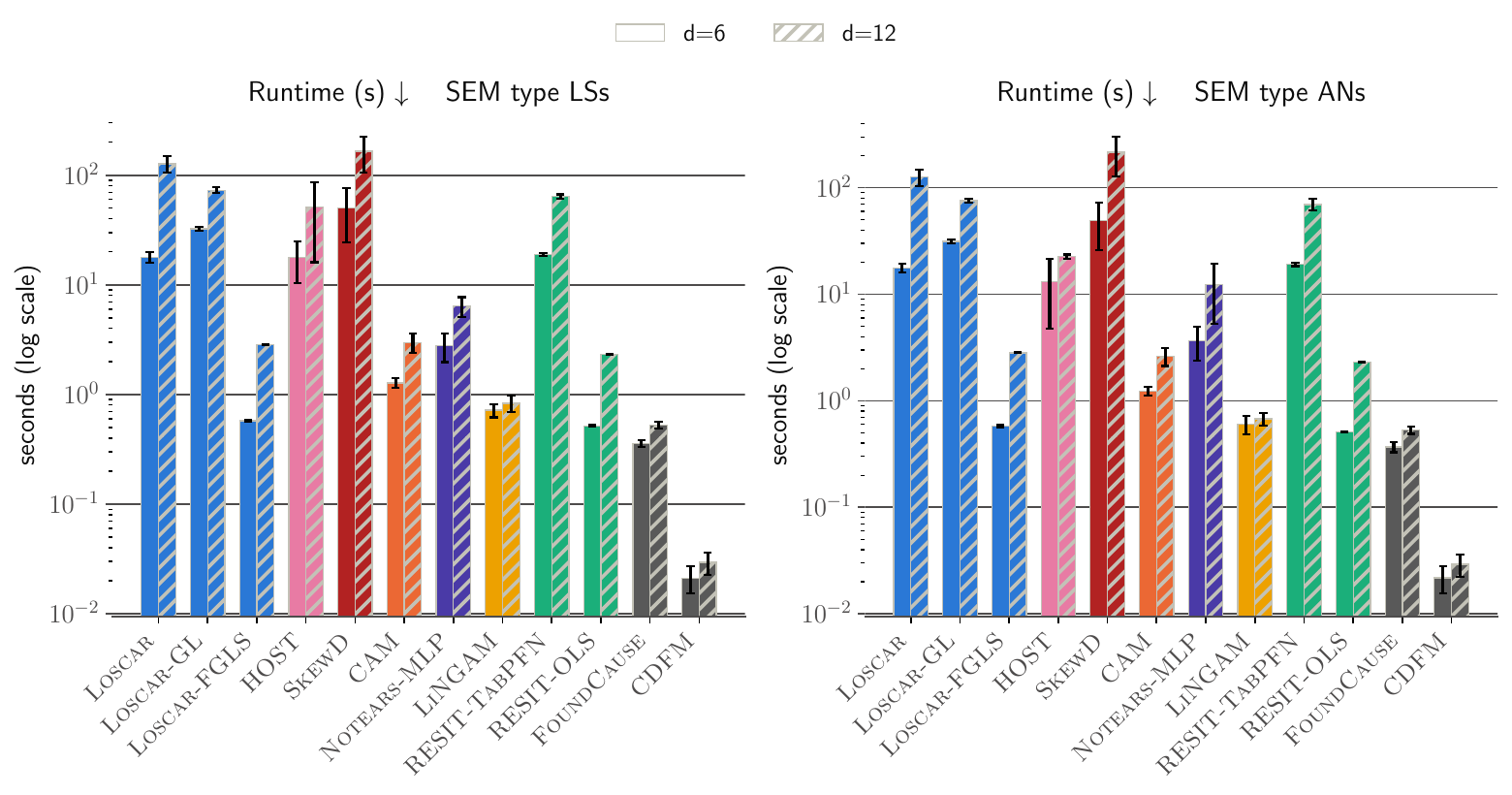}
    \caption{Average runtimes in seconds (s).}
    \label{fig:runtime}
\end{figure}

\subsubsection{Comparison under Homogeneous Noise Distributions} \label{app:homonoise}
In this section, we display and summarize the results that arise when the noise distribution in the generation of the synthetic multivariate data is held fixed at either exponential, Gaussian or skew-$t$ distribution. Apart from this nuance, the datasets were generated as in the heterogeneous noise case, where 100 datasets for both ANMs and LSNMs for $d=6,12$ were created. The parameterizations used are detailed in Section~\ref{app:datasets}.

We visualize the mean performances with respect to F1, order errors, SHD and SID as well as precision and recall for completeness using bar plots as in the previous sections. For LSNMs, the results are shown in Figures~\ref{fig:LSs_quad_exp}, \ref{fig:LSs_quad_gauss} and \ref{fig:LSs_quad_skewt} for the exponential, Gaussian, and skew-$t$ distribution, respectively. Equivalently, the ANM results are provided by Figures~\ref{fig:ANs_quad_exp}, \ref{fig:ANs_quad_gauss} and \ref{fig:ANs_quad_skewt}.

\paragraph{LSNM Setting}
Under exponential noise, \ourmethod and \textsc{\ourmethod-GL} constitute the strongest conventional baselines but are slightly outperformed by \textsc{FoundCause}. Similarly under skew-$t$ noise, \ourmethod and \textsc{\ourmethod-GL} perform best with regard to all metrics when $d=6$ with \textsc{SkewD} and \textsc{FoundCause} being competitive. However, when $d=12$, \textsc{SkewD} and \textsc{FoundCause} slightly outperform \ourmethod and \textsc{\ourmethod-GL}. Under Gaussianity, the assumptions of \textsc{SkewD} and \textsc{HOST} are fulfilled. This leads to notable performance gains in HOST compared to the other noise settings. \textsc{HOST} outperforms all variants of \ourmethod with respect to all metrics and dimensions and, together with \textsc{SkewD}, yields the best performances out of the conventional baselines. \textsc{FoundCause}, however, performs even better making it the most reliable choice here.

In summary, the results signalize that the quantile-based variants of \ourmethod are well-suited to handle asymmetric noise in LSNMs, as they achieve stable and strong performances across the different noise distributions. This is not surprising given the nonparametric nature of the method. In comparison, while also making no distributional assumptions, \textsc{\ourmethod-FGLS} performs poorly in the LSNM setting.

\paragraph{ANM Setting}
As in the mixed distribution setting, \textsc{\ourmethod-FGLS} outperforms the quantile variants of \ourmethod under ANMs across exponential and Gaussian noise. Interestingly, under the skew-$t$ distribution this is reversed. Moreover, the variants of \ourmethod remain competitive in the exponential settings, but are comparatively worse under Gaussian and skew-$t$ noise.

Under the exponential distribution, the best performances are achieved by \textsc{FoundCause}, \textsc{CDFM}, \textsc{CAM} and \textsc{RESIT-OLS} ($d=6$). All variants of \ourmethod are slightly worse but much more reliable than the other baselines. Compared to the LSNM case, \textsc{HOST} does not become a valid option in the Gaussian noise case, where \textsc{SkewD}, \textsc{CAM} as well as the amortized methods display the strongest results. This also applies to the skew-$t$ settings, where \textsc{CAM} performs best regarding every metric.

\begin{figure}[h!]
    \centering
    \includegraphics[width=0.99\linewidth]{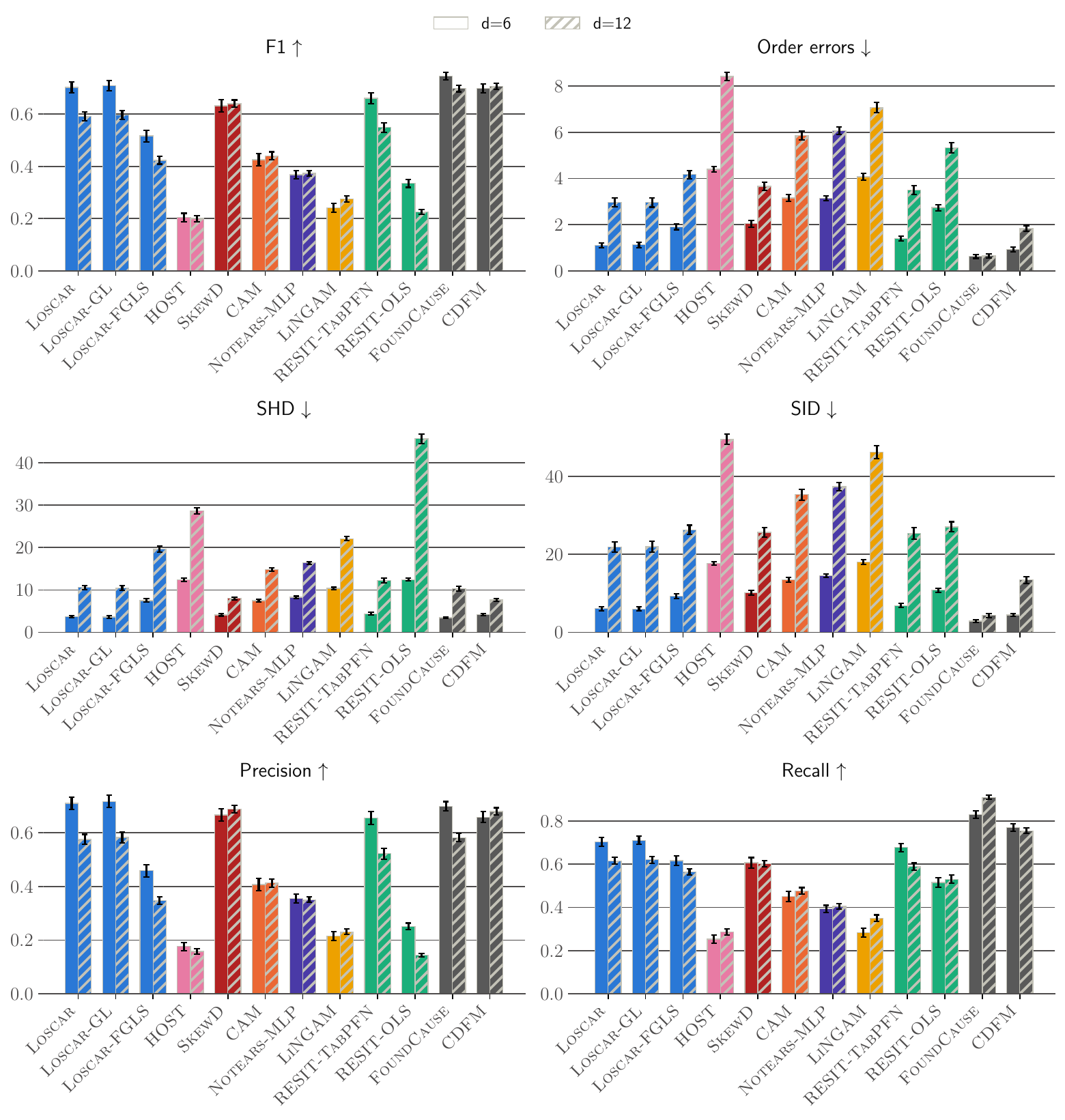}
    \caption{Comparison of \ourmethod and baselines in terms of F1, order errors, SHD and SID on LSNM data with exponential noise.}
    \label{fig:LSs_quad_exp}
\end{figure}

\begin{figure}[h!]
    \centering
    \includegraphics[width=0.99\linewidth]{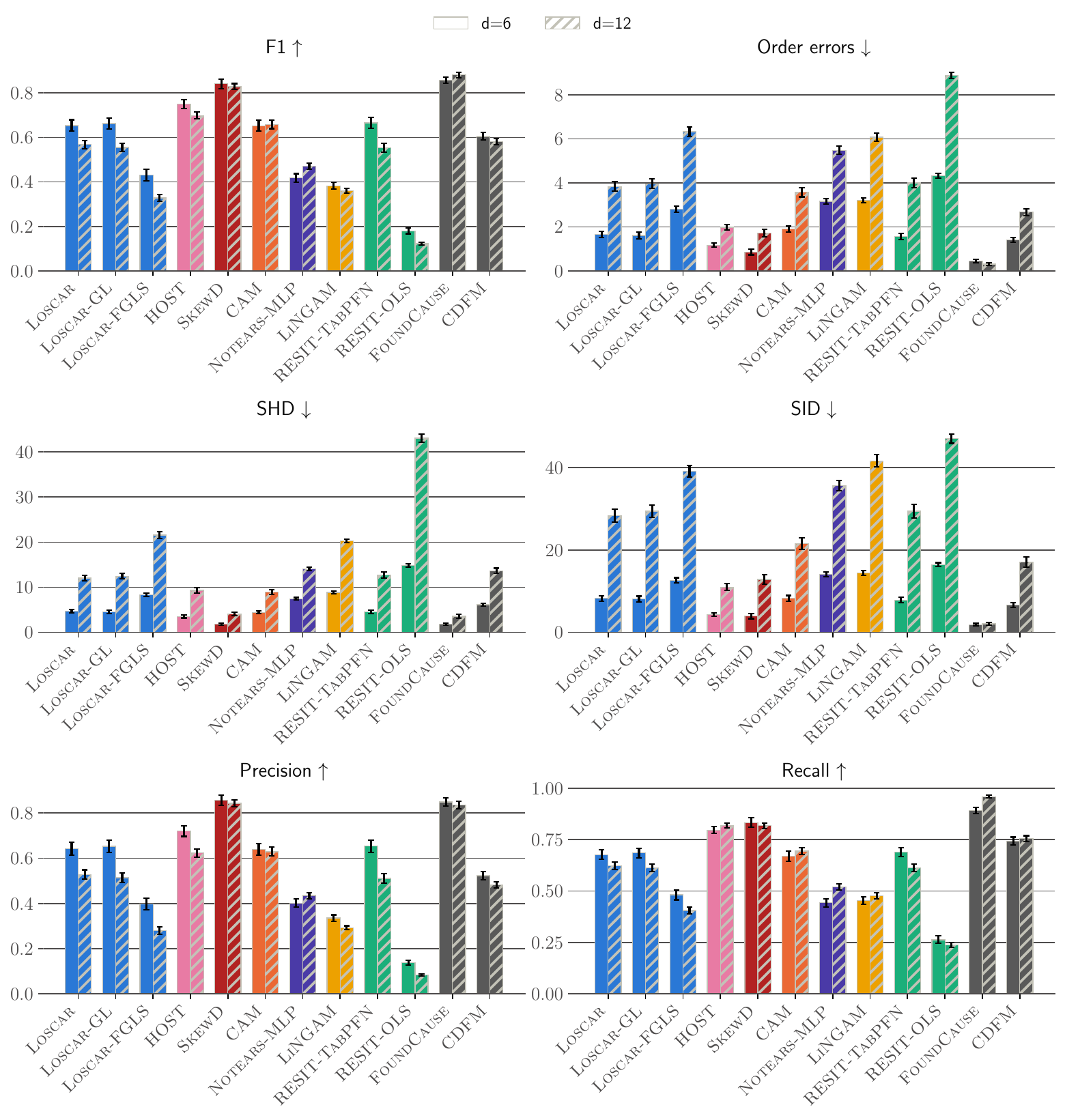}
    \caption{Comparison of \ourmethod and baselines in terms of F1, order errors, SHD and SID on LSNM data with Gaussian noise.}
    \label{fig:LSs_quad_gauss}
\end{figure}

\begin{figure}[h!]
    \centering
    \includegraphics[width=0.99\linewidth]{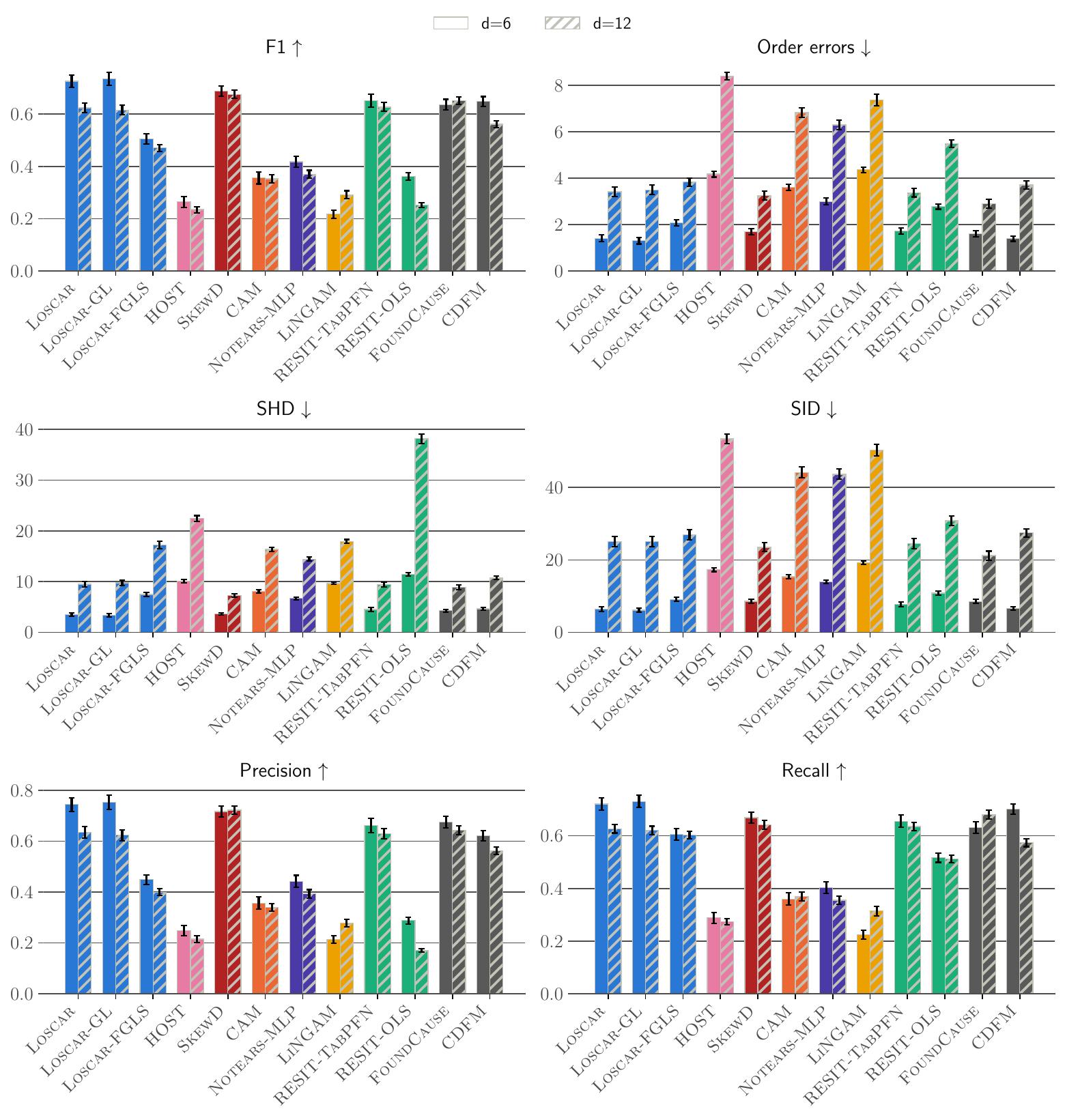}
    \caption{Comparison of \ourmethod and baselines in terms of F1, order errors, SHD and SID on LSNM data with skew-$t$ noise.}
    \label{fig:LSs_quad_skewt}
\end{figure}

\begin{figure}[h!]
    \centering
    \includegraphics[width=0.99\linewidth]{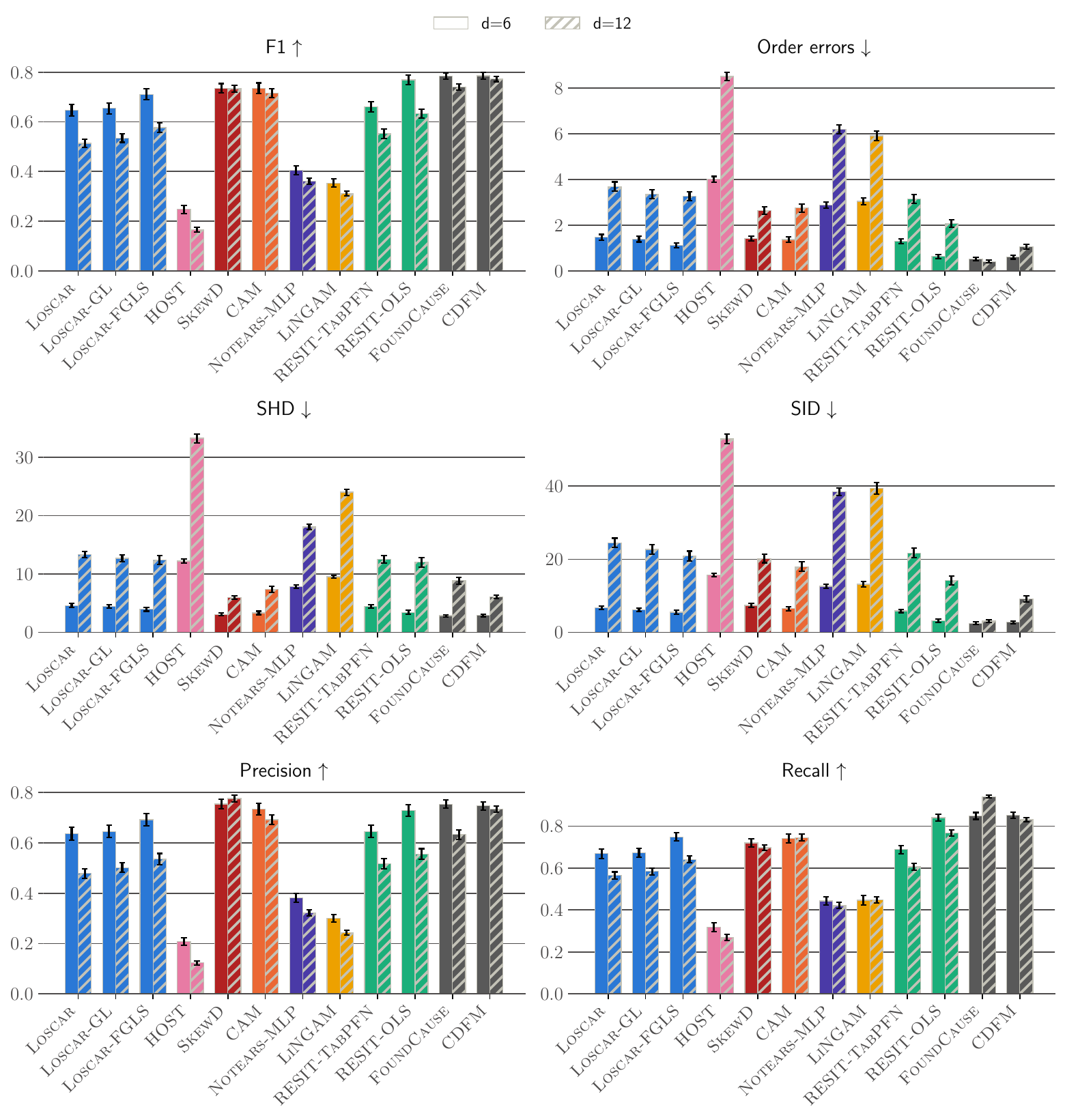}
    \caption{Comparison of \ourmethod and baselines in terms of F1, order errors, SHD and SID on ANM data with exponential noise.}
    \label{fig:ANs_quad_exp}
\end{figure}

\begin{figure}[h!]
    \centering
    \includegraphics[width=0.99\linewidth]{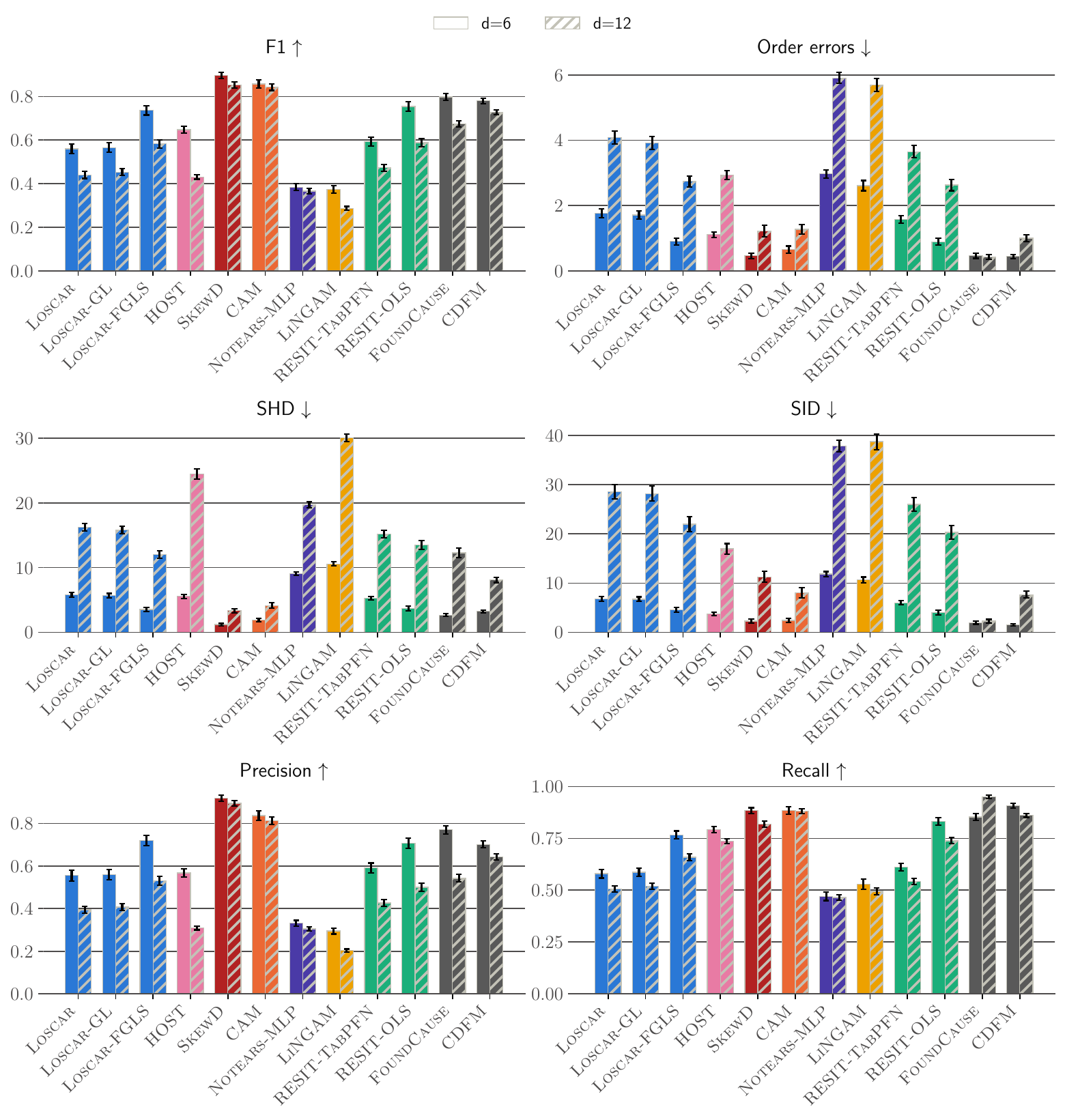}
    \caption{Comparison of \ourmethod and baselines in terms of F1, order errors, SHD and SID on ANM data with Gaussian noise.}
    \label{fig:ANs_quad_gauss}
\end{figure}

\begin{figure}[h!]
    \centering
    \includegraphics[width=0.99\linewidth]{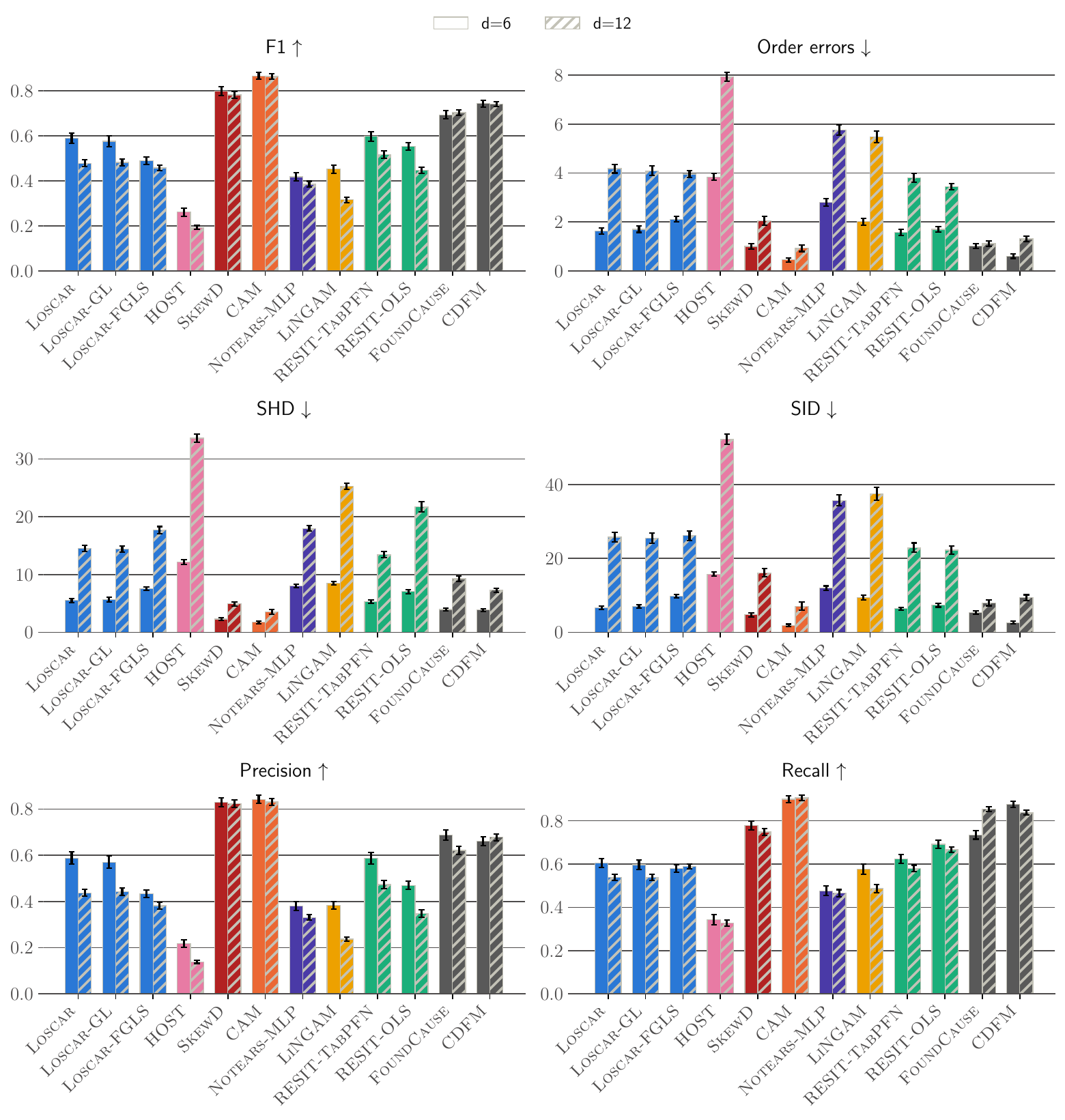}
    \caption{Comparison of \ourmethod and baselines in terms of F1, order errors, SHD and SID on ANM data with skew-$t$ noise.}
    \label{fig:ANs_quad_skewt}
\end{figure}


\end{document}

%% file: math_commands.tex
\usepackage{amsmath,amsfonts,bm}
\usepackage[dvipsnames]{xcolor}
\usepackage{cancel}

\definecolor{tu01}{HTML}{84B818}
\definecolor{tu02}{HTML}{D18B12}
\definecolor{tu03}{HTML}{1BB5B5}
\definecolor{tu04}{HTML}{F85A3E}
\definecolor{tu05}{HTML}{4B6CFC}
\definecolor{tu06}{HTML}{E3B505}
\definecolor{tu07}{HTML}{AF331D}
\definecolor{tu08}{HTML}{000000}
\definecolor{tu09}{HTML}{AAAAAA}
\definecolor{tu10}{HTML}{444444}
\definecolor{tu11}{HTML}{84B818}

\colorlet{tu01light}{tu01!33}
\colorlet{tu02light}{tu02!33}
\colorlet{tu03light}{tu03!33}
\colorlet{tu04light}{tu04!33}
\colorlet{tu05light}{tu05!33}
\colorlet{tu06light}{tu06!33}
\colorlet{tu07light}{tu07!33}
\colorlet{tu08light}{tu08!33}
\colorlet{tu09light}{tu09!33}
\colorlet{tu10light}{tu10!33}
\colorlet{tu11light}{tu11!33}

\colorlet{tu01midlight}{tu01!50}
\colorlet{tu02midlight}{tu02!50}
\colorlet{tu03midlight}{tu03!50}
\colorlet{tu04midlight}{tu04!50}
\colorlet{tu05midlight}{tu05!50}
\colorlet{tu06midlight}{tu06!50}
\colorlet{tu07midlight}{tu07!50}
\colorlet{tu08midlight}{tu08!50}
\colorlet{tu09midlight}{tu09!50}
\colorlet{tu10midlight}{tu10!50}
\colorlet{tu11midlight}{tu11!50}

\colorlet{tu01dark}{tu01!80!black}
\colorlet{tu02dark}{tu02!80!black}
\colorlet{tu03dark}{tu03!80!black}
\colorlet{tu04dark}{tu04!80!black}
\colorlet{tu05dark}{tu05!80!black}
\colorlet{tu06dark}{tu06!80!black}
\colorlet{tu07dark}{tu07!80!black}
\colorlet{tu08dark}{tu08!80!black}
\colorlet{tu09dark}{tu09!80!black}
\colorlet{tu10dark}{tu10!80!black}
\colorlet{tu11dark}{tu11!80!black}

\colorlet{lightgray}{gray!25}
\colorlet{midlightgray}{gray!50}
\colorlet{anthracite}{black!85}

\colorlet{tudo}{tu01}
\colorlet{tuorange}{tu02}
\colorlet{tudolight}{tu01light}

\definecolor{TUGreen}{HTML}{639A00} 

\def\1{\bm{1}}

\def\mN{{\bm{N}}}

\def\mV{{\bm{V}}}

\def\mX{{\bm{X}}}

\DeclareMathAlphabet{\mathsfit}{\encodingdefault}{\sfdefault}{m}{sl}
\SetMathAlphabet{\mathsfit}{bold}{\encodingdefault}{\sfdefault}{bx}{n}

\def\gC{{\mathcal{C}}}

\def\gE{{\mathcal{E}}}

\def\gG{{\mathcal{G}}}

\def\gS{{\mathcal{S}}}

\def\E{{\mathbb{E}}}

\newcommand{\Indep}{\mathop{\perp\!\!\!\perp}\nolimits}

\definecolor{lightgray}{gray}{0.85}
\definecolor{lightlightgray}{gray}{0.9}
\definecolor{C1}{HTML}{1F77B4}
\definecolor{C2}{HTML}{FF7F0E}
\definecolor{C3}{HTML}{2CA02C}
\definecolor{C4}{HTML}{D62728}
\definecolor{C5}{HTML}{9467BD}
\definecolor{Cg}{HTML}{7F7F7F}
\colorlet{C1light}{C1!90!white}
\colorlet{C2light}{C2!90!white}
\colorlet{C3light}{C3!90!white}
\colorlet{C4light}{C4!90!white}
\colorlet{C5light}{C5!90!white}
\colorlet{Cglight}{Cg!90!white}
\colorlet{C1lighter}{C1!40!white}
\colorlet{C2lighter}{C2!40!white}
\colorlet{C3lighter}{C3!40!white}
\colorlet{C4lighter}{C4!40!white}
\colorlet{C5lighter}{C5!40!white}
\colorlet{C1vlight}{C1!20!white}
\colorlet{C2vlight}{C2!20!white}
\colorlet{C3vlight}{C3!20!white}
\colorlet{C4vlight}{C4!20!white}
\colorlet{C5vlight}{C5!20!white}

\RequirePackage{todonotes}

